\documentclass{article}

\usepackage[preprint]{neurips_2023}

\usepackage{amssymb}            
\usepackage{mathtools}          
\usepackage{mathrsfs}           
\mathtoolsset{showonlyrefs}     
\usepackage{graphicx}           
\usepackage{subcaption}         
\usepackage[space]{grffile}     
\usepackage{url}                

\usepackage{booktabs} 
\usepackage{tikz} 

\usepackage{amsmath}
\usepackage{amssymb}
\usepackage{mathtools}
\usepackage{amsthm}

\usepackage{algorithm}
\usepackage{algpseudocode}

\newtheorem{theorem}{Theorem}[section]

\newtheorem{lemma}[theorem]{Lemma}

\newtheorem{assumption}{Assumption}

\usepackage[opnorm=op]{arash_macros}

\usepackage{natbib}
\usepackage{graphicx}

\usepackage{hyperref}       
\usepackage{url}            
\usepackage{booktabs}       
\usepackage{amsfonts}       
\usepackage{microtype}      
\usepackage{xcolor}         

\mathtoolsset{showonlyrefs}     

\usepackage{amsmath}
\usepackage{amsthm}
\usepackage{amssymb}
\usepackage{amsfonts}
\usepackage{mathrsfs}

\usepackage{pifont}

\newcommand{\bA}{{\boldsymbol A}}
\newcommand{\bB}{{\boldsymbol B}}

\newcommand{\bD}{{\boldsymbol D}}

\newcommand{\bG}{{\boldsymbol G}}

\newcommand{\bK}{{\boldsymbol K}}

\newcommand{\bI}{{\boldsymbol I}}
\newcommand{\bP}{{\boldsymbol P}}
\newcommand{\bQ}{{\boldsymbol Q}}

\newcommand{\bS}{{\boldsymbol S}}

\newcommand{\bU}{{\boldsymbol U}}
\newcommand{\bV}{{\boldsymbol V}}
\newcommand{\bW}{{\boldsymbol W}}
\newcommand{\bX}{{\boldsymbol X}}

\newcommand{\ba}{{\boldsymbol a}}
\newcommand{\bb}{{\boldsymbol b}}

\newcommand{\be}{{\boldsymbol e}}
\newcommand{\bbf}{{\boldsymbol f}}

\newcommand{\bh}{{\boldsymbol h}}

\newcommand{\bv}{{\boldsymbol v}}

\newcommand{\bx}{{\boldsymbol x}}
\newcommand{\by}{{\boldsymbol y}}

\newcommand{\BR}{\mathbb{R}}
\newcommand{\BE}{\mathbb{E}}

\newcommand{\BI}{\mathbb{I}}
\newcommand{\BP}{\mathbb{P}}

\newcommand{\btheta}{{\boldsymbol \theta}}
\newcommand{\bepsilon}{{\boldsymbol \epsilon}}

\newcommand{\bmu}{{\boldsymbol \mu}}

\newcommand{\bpsi}{{\boldsymbol \psi}}

\newcommand{\bLambda}{{\boldsymbol \Lambda}}

\newcommand{\bPsi}{{\boldsymbol \Psi}}
\newcommand{\bDelta}{{\boldsymbol \Delta}}

\newcommand{\bcg}{{\boldsymbol \cg}}

\newcommand{\bZeros}{{\boldsymbol 0}}

\newcommand{\mE}{{\mathcal{E}}}
\newcommand{\mF}{{\mathcal{F}}}
\newcommand{\mG}{{\mathcal{G}}}
\newcommand{\mH}{{\mathcal{H}}}

\newcommand{\mL}{{\mathcal{L}}}

\newcommand{\mN}{{\mathcal{N}}}
\newcommand{\mO}{{\mathcal{O}}}

\newcommand{\mV}{{\mathcal{V}}}

\newcommand{\mX}{{\mathcal{X}}}

\newcommand{\cg}{\textnormal{\textsl{g}}}

\newcommand{\relu}{{\mbox{ReLU}}}

\newcommand{\sigmoid}{{\mbox{sigmoid}}}
\newcommand{\degree}{{\mbox{deg}}}

\title{Graph Neural Thompson Sampling}

\author{Shuang Wu  \\
    shuangwu222@ucla.edu \\
    Department of Statistics and Data Science\\
    University of California, Los Angeles
    \And
    Arash A.~Amini  \\
    aaamini@stat.ucla.edu \\
    Department of Statistics and Data Science\\
    University of California, Los Angeles}

\newcommand\given{\,|\,}
\newcommand\fgnn{f_{\operatorname{GNN}}}
\newcommand\fmlp{f_{\operatorname{MLP}}}
\newcommand\gmlp{\bcg_{\operatorname{MLP}}}
\newcommand\kmlp{k_{\operatorname{MLP}}}
\newcommand\normalize{u}
\DeclareMathOperator{\poly}{poly}

\newcommand\Ec{\mathcal E}
\newcommand\Fc{\mathcal F}
\newcommand\Uc{\mathcal U}
\newcommand\Gc{\mathcal G}
\newcommand\Gb{\bar G}
\newcommand\ind{\mathbb I}
\newcommand\rhomax{\rho_{\max}}
\newcommand\deff{\tilde{d}}
\newcommand\rh{\widehat r}
\newcommand\bKh{\hat \bK}

\renewcommand\bS{\boldsymbol S}

\begin{document}

\maketitle

\begin{abstract}
We consider an online decision-making problem
with a reward function defined over graph-structured data. We formally formulate the problem as an instance of graph action bandit.
We then propose \texttt{GNN-TS}, a Graph Neural Network (GNN) powered Thompson Sampling (TS) algorithm which employs a GNN approximator for estimating the mean reward function and the graph neural tangent features for uncertainty estimation. 
We prove that, under certain boundness assumptions on the reward function, GNN-TS achieves a state-of-the-art regret bound  which is (1) sub-linear of order $\tilde{\mO}((\deff T)^{1/2})$ in the number of interaction rounds, $T$, and a notion of effective dimension $\tilde d$, and (2) independent of the number of graph nodes. Empirical results validate  that our proposed \texttt{GNN-TS} exhibits competitive performance and scales well on graph action bandit problems.
\end{abstract}

\vspace{-2mm}
\section{Introduction}\label{sec: introduction}
\vspace{-1mm}

Thompson Sampling~\citep{thompson1933likelihood} is a widely adopted and effective technique in sequential decision-making problems, known for its ease of implementation and practical success~\citep{chapelle2011empirical, kawale2015efficient, russo2018tutorial, riquelme2018deep}. 
The fundamental concept behind Thompson Sampling (TS)  is to compute the posterior probability of each action being optimal for the present context, followed by the selection of an action from this distribution.
Previous research has extended TS or developed variants of it to incorporate increasingly complex models of the reward function, such as Linear TS~\citep{agrawal2013thompson, abeille2017linear}, Kernelized TS~\citep{chowdhury2017kernelized}, and Neural TS~\citep{zhang2020neural}. 
However, these efforts have mainly focused on conventional data types. In contrast, the application of sequential learning to graph-structured data, such as molecular or biological graph representations, introduces unique challenges that merit further investigation.

Recently, there has been a growing interest in studying bandit optimization over graphs. Several researchers have initiated this line of work by addressing the challenge of encoding graph structures in bandit problems~\citep{gomez2018automatic, jin2018junction, griffiths2020constrained, korovina2020chembo}. More recently, Graph Neural Network (GNN) bandits have been proposed, which leverage expressive GNNs to approximate reward functions on graphs~\citep{kassraie2022graph}. Despite these advancements, the GNN bandits remain relatively unexplored compared to the extensive research on Neural bandits. Firstly, a formal formulation of this sequential graph selection problem is yet to be proposed. 
More importantly, there is a significant lack of comprehensive theoretical and empirical investigations regarding the use of TS in sequential graph selection.

\textbf{Contribution.} In this work, we address the online decision-making problem over graph-structured data by contributing a novel algorithm called \texttt{GNN-TS}. We begin by formulating the sequential graph selection as graph action bandit. We then propose Graph Neural Thompson Sampling, \texttt{GNN-TS}, to incorporate TS exploration with graph neural networks. We establish a regret bound for the proposed algorithm with sub-linear growth of order $\tilde{\mO}((\deff T)^{1/2})$ with respect to the effective dimension $\deff$ and the number of interaction round $T$, and independent of the number of graph nodes. Finally, we corroborate the analysis with an empirical evaluation of the algorithm in simulations. Experiments show that \texttt{GNN-TS} yields competitive performance and scalability, compared to the state-of-the-art baselines, underscoring its practical value in addition to its strong theoretical guarantees.

\textbf{Notations.} 
Let $[n] = \{1,2,...,n\}$. For a  set or event $\mE$, we denote its complement as $\bar{\mE}$. $\bI_n \in \BR^{n \times n}$ is the identity matrix.  For a matrix $\bA$, $\bA_{i*}$ and $\bA_{*j}$ denote its $i$-th row and $j$-th column, respectively. $\lambda_{\max}(\bA)$ and $\lambda_{\min}(\bA)$ represents the maximum and minimum eigenvalues of the matrix $\bA$. For any vector $\bx$ and square matrix $\bA$, $\norm{\bx}_{\bA}=\sqrt{\bx^{\top} \bA \bx}$.  We denote the history of randomness up to (but not including) round $t$ as $\mF_t$ and write $\BP_t(\cdot) := \BP(\,\cdot\, \given \mF_t)$ and $\BE_t(\cdot) := \BE[\,\cdot\, \given \mF_t]$ for the conditional probability and expectation given $\mF_t$. We use $\lesssim$ and big-$O$, to denote ``less than'', up to a constant factor.
We further use $\Tilde{\mO}(\cdot)$ for big-$O$ up to logarithmic factor. 

\vspace{-1mm}
\section{Related Works} \label{sec: related_works}
\vspace{-2mm}

\textbf{Graph Bandit.} 
Multiple works have studied graph bandit problems, which can be classified into two categories: graph as structure across arms and graph as data. 
Most research focuses on the former category, starting from spectral bandit~\citep{kocak2014spectral, kocak2020spectral} to graphical bandit~\citep{liu2018analysis, yu2020graphical, gou2023stochastic, toni2023online}. Within this field, bandit problems with graph feedback have garnered significant attention~\citep{tossou2017thompson, dann2020reinforcement, chen2021understanding, kong2022simultaneously}, where learners observe rewards from selected nodes and their neighborhoods. The primary focus of these works have been improving sample efficiency~\citep{bellemare2019geometric, waradpande2020deep, ide2022targeted}, with some assuming that payoffs are shared according to the graph Laplacian~\citep{esposito2022learning, lee2020closer, lykouris2020feedback, thaker2022maximizing, yang2020laplacian}. 
While the existing literature primarily aims to optimize over geometrical signal domains, our work focuses on optimization within graph domains. Specifically, we investigate the online graph selection problem, aligning with the second category of research that considers the entire graph as input data. 
A related recent work~\citep{kassraie2022graph} proposed a GNN bandit approach with regret bound based on information gain and an elimination-based algorithm. In contrast, our work explores regret bound based on the effective dimension and builds upon the foundation of Thompson Sampling.
This second category of research also encompasses empirical works~\citep{upadhyay2020graph, qi2022neural, qi2023graph}, particularly those centered around molecule optimization~\citep{wang2023graph_a, wang2023graph_b}.

\textbf{Neural Bandit.}
Our work contributes to the research on neural bandits, where deep neural networks are utilized to estimate the reward function. The work of~\cite{zahavy2019deep, xu2020neural} investigated the Neural Linear bandit, while \cite{zhou2020neural} developed Neural Upper Confidence Bound (UCB), an extension of Linear UCB. \cite{zhang2020neural} adapted TS with deep neural networks, proposing Neural TS. \cite{dai2022sample} makes improvements to neural bandit algorithms to overcome practical limitations. \cite{nguyen2021offline} explores neural bandit in an offline contextual bandit setting and ~\citep{gu2024batched} examines batched learning for neural bandit. Our work can be seen as an extension of Neural TS \citep{zhang2020neural}, incorporating significant improvements such as the utilization of graph neural tangent kernel and a distinct definition of effective dimension.



\vspace{-2mm}
\section{Problem Formulation and Methodology}
\label{sec: problem formulation}
\vspace{-2mm}

\subsection{Graph Action Bandit Problem}
\label{subsec: graph action bandit problem}
\vspace{-1mm}

We consider an online decision-making problem in which the learner aims to optimize an unknown reward function by sequentially interacting with a stochastic environment. We identify the actions with graphs from an action space $\mG$ and assume that the size of this action space, denoted as $|\mathcal{G}|$, is finite. At time $t \in [T]$, the learner selects a graph $G_t$ from the action space $\mG_t \subset \mG$. The learner then observes a noisy reward $y_t = \mu(G_t)+\eps_t$ where $\mu: \mG \rightarrow \BR$ is the true (unknown) reward function and $\{\eps_t\}_{t\in[T]}$ are i.i.d zero-mean sub-gaussian noise with variance proxy $\sigma^2_{\eps}$. The goal of the learner is to maximize the expected cumulative reward in $T$ rounds, which equivalently entails minimizing the expected (pseudo-)regret denoted as $R_T = \sum_{t=1}^T \BE[\mu(G^{*}_t) - \mu(G_t)]$ where $G^{*}_t = \argmax_{G\in \mG_t} \mu(G)$ represents the optimal graph at time~$t$.

The graph space $\mG$ is a finite set of undirected graphs with at most $N$ nodes. Note that the graphs with less than $N$ nodes can be treated by adding auxiliary isolated nodes with no features. We denote an undirected attributed graph with $N$ nodes as $G = (\mathbf{X}, \mathbf{A})$, where $\mathbf{X} \in \mathbb{R}^{N \times d}$ represents the feature matrix with $d$ features, and $\mathbf{A} \in \{0, 1\}^{N \times N}$ is the unweighted adjacency matrix. The rows of $\mathbf{X}$ correspond to node features. The size of the node set of a graph $G$ is denoted as $|\mathcal{V}(G)| \leq N$. 

Graph action bandit has several applications such as chemical molecules optimization. Consider the graph structures representing the molecules \citep{weininger1988smiles} and rewards are molecular properties. The goal is to sequentially recommend the optimal molecules for experimental testing.

\subsection{Graph Neural Network Model }
\label{subsec: graph neural network model}
\vspace{-1mm}
We propose to learn the unknown reward function $\mu(\cdot)$ by fitting a Graph Neural Network (GNN). We consider a relatively simple GNN architecture where the output of a single graph convolution layer is normalized (to unit $\ell_2$ norm) and passed through a multilayer perceptron (MLP). A single-layer graph convolution can be compactly stated as $\bA \bX$ using the adjacency matrix $\bA$ of the network. Additionally, we normalize each row of the resulting matrix to have a unit $\ell_2$ norm. Letting $\normalize(\bx) = \bx / \norm{\bx}_2$ denote the normalization operator, the aggregated feature of node $i$ in a graph $G$ is 
$\bh_i^G= \normalize ( (\bA \bX)_{i*} ) = \normalize ( \sum_{j \in \mN_i} \bX_{j*})$
where $\mN_j$ is the neighborhood of node $j$. 
Our GNN also consists of an $L$-layer $m$-width MLP neural network $\fmlp$ which is defined recursively as follows
\begin{equation}
\label{eq: MLP}
\begin{aligned}
    f^{(1)}(\bh_i^G) &= \bW^{(1)} \bh_i^G, \quad i \in [N], \\
    f^{(l)}(\bh_i^G) &= \frac{1}{\sqrt m}
    \bW^{(l)} \relu(f^{(l-1)}(\bh_i^G)), \quad 2\leq l \leq L,\\
    \fmlp(\bh_i^G;\btheta) &= f^{(L)}(\bh_i^G).
\end{aligned}
\end{equation} \vspace{-2mm}

Here, $\relu(\cdot)=\max(\cdot, 0)$,
$\bW^{(1)} \in \BR^{m \times d}$, $\bW^{(L)} \in \BR^{1 \times m}$,  $\bW^{(l)} \in \BR^{m \times m}$ for any $1<l<L$ are weight matrices of the MLP and $\btheta := (\bW^{(1)},\dots, \bW^{(L)}) \in \BR^p$ is the collection of parameters of the neural network where $p=dm+(L-2)m^2+m$. Our GNN model to estimate the reward function is
\vspace{-1mm}
\begin{equation}
\label{eq: GNN}
\fgnn(G; \btheta) := \frac{1}{N} \sum_{i=1}^{N} \fmlp(\bh_i^G;\btheta).
\end{equation}\vspace{-2mm}

The gradient of $\btheta \mapsto \fgnn(G; \btheta)$ denoted as $\bcg(G;\btheta) := \nabla_{\btheta} \fgnn(G; \btheta)$
will play a key role in uncertainty quantification, which will be discussed in Section \ref{subsec: graph neural thompson sampling}. 
The GNN model \eqref{eq: GNN} is trained by minimizing the mean-squared loss with $\ell_2$ penalty, described concretely in \eqref{eq: minimize l2 loss}.  
A hyperparameter $\lambda$ is used to tune the strength of $\ell_2$ regularization.
For the simplicity of exposition, in the theoretical analysis, we solve the optimization via gradient descent with learning rate $\eta$, total number of iterations $J$ and initialize parameters $\btheta_0$ such that $\fgnn(G;\btheta_0)=0$ for all $G \in \mG$, which can be fulfilled based on the work of \cite{zhou2020neural, kassraie2022neural}.




\subsection{Graph Neural Thompson Sampling}
\label{subsec: graph neural thompson sampling}
\vspace{-1mm}


We adapt Thompson Sampling (TS) for graph exploration, due to its robust performance in balancing exploration against exploitation. Algorithm~\ref{alg:GraphNeuralTS} outlines our proposed GNN Thompson sampling, following the idea of \texttt{NeuralTS} in \cite{zhang2020neural}. The key step is the sampling of an estimated reward mean $\rh_t(G)$ for each graph $G$ in the action space at time $t$, from a normal distribution as in equation~\eqref{eq:TS:sampling}. The mean of the normal distribution in~\eqref{eq:TS:sampling} is our current estimate, $\fgnn(G; \btheta_{t-1})$, of the true mean reward for graph $G$ (i.e., $\mu(G)$). This estimate is obtained by fitting the GNN to all the past data as in~\eqref{eq: minimize l2 loss}. The variance of the normal distribution $\nu^2 \sigma_t^2(G)$ is our current measure of uncertainty about the true reward of graph $G$. Note that
\vspace{-1mm}
\begin{equation}
    \sigma^2_{t}(G) = 
   \frac1{m}
   \norm{\bcg(G; \btheta_{t-1})}_{\bU_{t-1}^{-1}}^2 
    \quad \text{where} \quad
    \bU_{t-1} = \lambda \bI_p + \frac1m \sum_{i=1}^{t-1}  
    \bcg (G_{i}; \btheta_{i-1}) \bcg(G_{i}; \btheta_{i-1})^{\top}.
    \label{eq:var:2}
\end{equation} \vspace{-6mm}

The rationale behind $\sigma_t^2(G)$ comes from a linear approximation to $\fgnn(G; \btheta)$.  In particular, the idea is that~\eqref{eq: minimize l2 loss} approximately looks like a linear ridge  regression problem, with features $\{\bcg(G_i; \btheta_i) / \sqrt{m}\}_{i \in [t]}$. The expression~\eqref{eq:var:2} is then the familiar estimated covariance matrix from linear bandits 
after we make this identification. 
This approximation can be made rigorous via the neural tangent kernel idea, as discussed in Section~\ref{sec: regret bound}.


The sampled reward mean $\rh_{t}(G)$ is the index for decision-making. The learner selects the graph with the highest index, i.e., $G_t = \argmax_{G \in \mG} \rh_{t}(G)$.  The randomness in $\rh_t(G)$, due to the positive variance of the sampling distribution, is what allows TS to efficiently explore the action space. We want the uncertainty, as captured by $\sigma_t^2(G)$ not to be too small early on, to allow for effective exploration, but not too large either to miss out on the optimal choice too often. Lemma~\ref{lemma: high probability bound for event for exploration} in Section~\ref{sec: proof for the regret bound} captures the two sides of this trade-off in our theory.


It is worth noting that our proposed Algorithm~\ref{alg:GraphNeuralTS} is not exact TS. In our approach, \eqref{eq:TS:sampling} serves as an \emph{approximation to} a posterior for mean reward function, rather than a true posterior. 
The difference between our proposed method and 
an exact Bayesian method will be smaller if the GNN model is better approximated by a linear model. 

Lastly, we note that $\rh_{t}(G)$ is also referred to as the perturbed mean reward, as it can be expressed as: 
$\rh_{t}(G) = \fgnn (G; \btheta_{t-1}) + \nu \sigma_{t}(G) z$ where $z \sim \mN(0, 1)$.
This perturbed reward includes both the estimated part ($f_{\text{gnn}} (G; \boldsymbol{\theta}_{t-1})$) and the random perturbation part ($\nu \sigma_{t}(G) \cdot z$). The use of perturbations for exploration has been shown to be a strong strategy in previous works \citep{kim2019optimality, kveton2019perturbed_a}. Algorithm~\ref{alg:GraphNeuralTS} can be summarized as greedily selecting the graph with the highest \emph{perturbed} mean reward.

\begin{algorithm}[tb]
\caption{Graph Neural Thompson Sampling (\texttt{GNN-TS})}\label{alg:GraphNeuralTS}
\begin{algorithmic}[1]
   \State {\bfseries Input:} $T$, $\lambda$, $\nu$
   \State Initialization: $\btheta_{0}$, $\bU_{0} = \lambda \bI_p$.
   \For{$t = 1,...,T$}
        \State Compute $\sigma_t^2(G):=
        \frac{1}{m} 
        \norm{\bcg(G; \btheta_{t-1})}_{\bU_{t-1}^{-1}}^2$ 
        and sample
        \vspace{-3mm}
        \begin{equation}
        \label{eq:TS:sampling}
            \rh_t(G) \sim \mN 
            \bigl( \fgnn(G; \btheta_{t-1}), 
            \nu^2 \sigma_t^2(G) \bigr), 
            \quad \text{for all $G \in \mG_t$.}
        \end{equation}
        \vspace{-5mm}
        \State Select graph $G_{t} = \argmax_{G \in \mG_t} 
        \rh_t(G)$, 
        and collect reward $y_{t} := 
        \mu(G_t)+\eps_t$.
        \State Update uncertainty estimate as
        \vspace{-3mm}
        \begin{equation}
        \label{eq: update U}
            \bU_t = \bU_{t-1} + \bcg (G_{t}; \btheta_{t-1}) \bcg(G_{t}; \btheta_{t-1})^{\top}/m.
        \end{equation}
        \vspace{-5mm}
        \State Update the parameter estimate as
        \vspace{-3mm}
        \begin{equation}
        \label{eq: minimize l2 loss}
            \btheta_t = \argmin_{\btheta} \frac{1}{2t}\sum_{i=1}^t \bigl( \fgnn(G_i; \btheta) - y_i \bigr)^2 +  \frac{m\lambda}{2} \norm{\btheta}_2^2.
        \end{equation}
        \vspace{-5mm}
   \EndFor
\end{algorithmic}
\end{algorithm}

\vspace{-1mm}
\section{Regret Bound for \texttt{GNN-TS}}
\label{sec: regret bound}
\vspace{-1mm}

\textbf{Graph Neural Tangent Kernel.}
Let us briefly review the idea of graph neural tangent kernel (GNTK)~\citep{kassraie2022graph} which is based on the neural tangent kernel (NTK) of~\citep{jacot2018neural}. 
The tangent kernel on graph space $\mG$, induced by initialization $\btheta_0$, is defined as the inner product of the gradient at initialization, i.e $\tilde{k}(G, G^{\prime}) := \bcg(G;\btheta_0)^{\top} \bcg(G^{\prime};\btheta_0)$ for any $G, G^{\prime} \in \mG$. The GNTK is the limiting kernel of $\tilde{k}(G, G^{\prime})/m$. We define the finite-width (empirical) and infinite-width GNTK as
\begin{equation}
\label{eq: definition of gntk}
\begin{aligned}
    \hat{k}(G, G^{\prime}) :=  \frac1m \ip{\bcg(G;\btheta_0), \bcg(G';\btheta_0)}, \quad 
    k(G,G') := \lim_{m \to \infty}  \frac1m \ip{\bcg(G;\btheta_0), \bcg(G';\btheta_0)}.
\end{aligned}
\end{equation} \vspace{-4mm}


We assume the reward function falls within the RKHS corresponding to the GNTK $k$ defined in \eqref{eq: definition of gntk}. 
Define $\bK \in \BR^{|\mG|\times |\mG|}$ as the GNTK matrix with entries $k(G, G^{\prime})$ for all $G, G^{\prime} \in \mG$ and $\bmu = (\mu(G))_{G \in \mG} \in \BR^{|\mG|}$ as the reward function vector. 
The kernel matrix $\bK$ is positive definite with maximum eigenvalue $\rho_{\max}:= \lambda_{max}(\bK)$ and minimum eigenvalue $\rho_{min}:= \lambda_{min}(\bK)$. We also define the finite-width GNTK matrix $\bKh \in \reals^{|\Gc| \times |\Gc|}$ with entries $\hat{k}(G, G^{\prime})$ for all $G, G^{\prime} \in \mG$ and maximum eigenvalues $\hat{\rho}_{\max} = \lambda_{\max}(\bKh)$. Note that $\bKh \to \bK$ as $m \to \infty$. 



\textbf{Effective Dimension.} We define the effective dimension $\deff$ of the GNTK matrix $\bK$ with regularization $\lambda$ as
\begin{equation}
\label{eq:def of effective dimension}
    \deff := 
    \frac{\log \det (\bI_{|\mG|} + T \bK / \lambda)}{\log(1 + T \rho_{\max} /\lambda)}.
\end{equation} \vspace{-2mm}

This quantity, which appears in our regret bound, measures the actual underlying dimension of the reward function space as seen by the bandit problem~\citep{valko2013finite, bietti2019inductive}. Our definition is adapted from~\citep{yang2020reinforcement}. The key difference is that our $\deff$ does not directly depend on $|\mG|$, which is replaced by $\rho_{\max}$, compared to the definition in \citep{zhang2020neural}. Our definition is the ratio of the sum over the maximum of the sequence of log-eigenvalues of matrix $\bI_{|\mG|} + T \bK / \lambda$. As such, it is a robust measure of matrix rank. In particular, we always have $\deff \le |\Gc|$. Moreover, previous work on GNN bandit~\citep{kassraie2022graph} utilized the notion of information gain which we replace with the related, but different, notion of effective dimension $\deff$.

We will make the following assumptions:

\begin{assumption}[Bounded RKHS norm for Reward]
\label{assumption: bounded rkhs norm}
The reward function $\mu$ has $R$-bounded RKHS norm with respect to a positive definite kernel $k$: $\norm{\mu}_{k} = \sqrt{\bmu^{\top} \bK^{-1} \bmu} \leq R$. 
\end{assumption}

\begin{assumption} [Bounded Reward Differences]
\label{assumption: bounded reward differences} 
Reward differences between any graph in action space are bounded. 
Formally, $\forall G, G^{\prime} \in \mG$: $|\mu(G) - \mu(G^{\prime})| \leq B$, for some $B \ge 1$.
\end{assumption}

\begin{assumption}[Subgaussian Noise]
\label{assumption: subgaussian noise}
Noise process $\{ \epsilon_t\}_{t \in [T]}$ satisfies $\BE_{t-1}[e^{\eta \epsilon_t}] \leq e^{\sigma_{\epsilon}^2 \eta^2 / 2}, \forall \eta > 0$. 
\end{assumption}

Assumption~\ref{assumption: bounded rkhs norm} aligns with the regularity assumption commonly found in the kernelized and neural bandit literature \citep{srinivas2009gaussian, chowdhury2017kernelized, kassraie2022neural}. Assumption~\ref{assumption: bounded reward differences} implies that instantaneous regret is bounded: $|\mu(G_t^{*}) - \mu(G_t)| \leq B$ for all $t \in [T]$ and Assumption~\ref{assumption: subgaussian noise} is the conditional subgaussian assumption for stochastic process $\{ \epsilon_t\}_{t \in [T]}$. 


We are now ready to state our main result.
Recall that $N$ is the maximum number of (graph) nodes and $L$ the depth of MLP and $m$ its width. 


\begin{theorem}
\label{theorem: regret upper bound for TS}
Suppose
Assumption~\ref{assumption: bounded rkhs norm},\ref{assumption: bounded reward differences} and \ref{assumption: subgaussian noise} hold.
For a fixed horizon $T \in \nats$, let 
\begin{equation}
\begin{aligned}
    m &\ge \poly\bigl(T, L, |\mG|, \lambda^{-1}, R, \sigma_{\eps}, \rho_{\min}^{-1}, \log(T L N |\mG|) \bigr) \notag \\
    \nu &\gtrsim 1 + \sigma_\eps \sqrt{\deff \log T}  + \sqrt \lambda  R, \quad 
    \lambda \gtrsim (\sigma_\eps^2 + R^2)^3 + \rhomax \label{eq:m, nu:lambda}
\end{aligned}
\end{equation} 
and learning rate $\eta \leq (\tilde{C} mL+m\lambda)^{-1}$, for some constant $\tilde{C} $. Then, the regret of Algorithm~\ref{alg:GraphNeuralTS} is bounded as
\begin{align*}
    R_T \le C \, B \sqrt{\deff\, T \log (T|\Gc|) \cdot  \log(2+T \rhomax/\lambda) }
\end{align*} 
for some universal constant $C > 0$. 
\end{theorem}

The order of regret upper bound in Theorem~\ref{theorem: regret upper bound for TS}, $\tilde{\mO}((\deff T)^{1/2})$ matches the state-of-the-art regret bounds in the literature of Thompson Sampling \citep{agrawal2013thompson, chowdhury2017kernelized, kveton2020randomized, zhang2020neural}. As in \citep{kassraie2022graph}, our regret bound is independent of $N$, indicating that \texttt{GNN-TS} is valid for large graphs. 
Moreover, for low complexity reward functions of effective dimension $\deff = O(1)$, the regret scales as $\sqrt{\log |\Gc|}$ in the size of the action space, showing the robust scalability of \texttt{GNN-TS}.

\vspace{-2mm}
\section{Proof of the Regret Bound}
\label{sec: proof for the regret bound}
\vspace{-2mm}
Similar to the previous literature, the key is to to obtain probabilistic control on the `discrepancy’ of the policy in \texttt{GNN-TS} 
Consider the following events
\vspace{-1mm}
\begin{align}
    \mE_t^{\mu} 
     &:= \Bigl\{
     \bigl|\fgnn(G;\btheta_{t-1}) - \mu(G)\bigr| 
     \;\leq\; c_t^\mu(G),  \; \text{for all}\; G \in \mG_t \Bigr\} \\
     \mE_t^{\sigma} &:= 
     \Bigl\{\bigl|\rh_{t}(G) - \fgnn(G;\btheta_{t-1}) \bigr| \;\leq\; 
     c_t^\sigma(G) 
     ,\; \text{for all}\; G \in \mG_t  \Bigr\} \\
     \mE_t^{a} &:= \Bigl\{
     \rh_{t}(G^{*}_t)-\fgnn(G^{*}_t;\btheta_{t-1}) \;>\; 
     \nu \sigma_t(G^*_t) \Bigr\}
\end{align} \vspace{-2mm}

where  $c_t^\mu(G) := \nu \sigma_t(G) + \eps(t,m)$ and  $c_t^\sigma(G) :=  \nu \sigma_{t}(G) \sqrt{ 2\log (t^2 |\mG_t|)}$ as well as $\eps(t,m) = (C_0 \nu L^{9/2})\,  m^{-1/6}\sqrt{\log m } \cdot t$
and $C_0$ is some universal constant. Events $\mE^{\mu}_t$ and $\mE^{\sigma}_t$ control the discrepancies with constants $c_{t}^{\mu}(G)$ and $c_{t}^{\sigma}(G)$ respectively: $c_{t}^{\mu}(G)$ is bounding the estimation discrepancy while $c_{t}^{\sigma}(G)$ is bounding the exploration discrepancy. 
Note that event $\mE_t^{a}$ is only for $G^{*}_t$, the optimal graph at round $t$.

\vspace{-2mm}
\subsection{Estimation Bound ($\mE^{\mu}_t$)}
\label{subsec: estimation bound}

The following lemma ensures that event $\mE^{\mu}_t$ happens  with high probability. 

\begin{lemma} 
\label{lemma: high probability bound for event mu}
Fix $\delta \in (0,1)$. 
For $m \ge \poly(R, \sigma_{\eps}, L, |\mG|, \lambda^{-1}, \rho_{\min}^{-1}, \log(T L N |\mG| / \delta))$  and 
$(\nu, \lambda, \eta)$ satisfying conditions of Theorem~\ref{theorem: regret upper bound for TS},
we have 
$\BP ( \mE_t^{\mu}) \geq 1-\delta/T$.
\end{lemma} 

In other words, given a large enough width of the GNN ($m$) and a small enough learning rate ($\eta$), there is a high probability upper bound for the estimation error $|\fgnn(G;\btheta_{t-1}) - \mu(G)|$. This Lemma~\ref{lemma: high probability bound for event mu} also gives an approximate upper confidence bound similar to \texttt{GNN-UCB} in \citep{kassraie2022graph}: $\mu(G) \leq \fgnn(G;\btheta_{t-1}) + \nu \sigma_t(G) + \eps(t,m).$
Since $\eps(t,m)$ is negligible for large $m$, the approximate upper confidence bound, $\fgnn(G;\btheta_{t-1}) + \nu \sigma_t(G)$ is used as the index for \texttt{GNN-UCB}. 
Note that this lemma controls the estimation error produced by GNNs, hence applicable to all GNN bandit algorithms using model~\eqref{eq: GNN}. Our $c_{t}^{\mu}(G) = \nu \sigma_{t}(G) + \varepsilon(t,m)$ is similar in form to that of~\cite{zhang2020neural} which is different from the earlier analysis of TS in \cite{agrawal2013thompson}. 

\vspace{-1mm}
\subsection{Exploration Bound ($\mE^{\sigma}_t, \mE^{a}_t$)}
\label{subsec: exploration bound}

We also need event $\mE_t^{\sigma}$ to quantify the level of exploration achieved by the algorithm.
Intuitively, $\mE_t^{\sigma}$ ensures our exploration is moderate. 
On the other hand, indicated by the regret analysis in \citep{kveton2019perturbed_b}, instead of controlling the exploration independently, the relation between two sources of explorations needs to be considered because this relation is critical for finding the optimal action. To meet such observation, we define an extra "good" event for anti-concentration on the optimal actions, which is $\mE^{a}_t$. Under event $\mE_t^{a}$, the policy index $\rh_{t}(G^{*}_t)$ of the optimal graph has the higher future positive exploration, which guides the learner to have higher chance to pick the optimal graph. A formal lemma for exploration discrepancy using TS is given as below: 

\begin{lemma}
\label{lemma: high probability bound for event for exploration}
For \texttt{GNN-TS}, for all $t \in [T]$, we have $\pr_t(\bar \mE_t^{\sigma} ) \le t^{-2}$ and $\pr(\Ec_t^a) \ge (4 e\sqrt \pi)^{-1}$.
\end{lemma} 

Lemma~\ref{lemma: high probability bound for event for exploration} shows that \texttt{GNN-TS} has a positive probability of moderate exploration of the optimal arm, which is beneficial to regret reduction.

\subsection{Proof of Theorem \ref{theorem: regret upper bound for TS}}
\label{subsec: proof for main theorem}


Let $\Delta_t:=\mu(G_t^{*}) - \mu(G_t)$ be the instantaneous regret. 
We will need two additional lemmas:

\begin{lemma}[One Step Regret Bound]
\label{lemma: one step regret bound}
Assume the same as Theorem \ref{theorem: regret upper bound for TS}. Suppose $\BP_t( \mE_t^{a}) - \BP_t(\bar{\mE}_t^{\sigma}) > 0$. Then for any $t \in [T]$, almost surely, 
\begin{equation}
\begin{aligned}
    \BE_t[\Delta_t \ind_{\Ec_t^\mu}]  
    \leq
    \BI_{\mE_t^\mu} \cdot \Bigl\{ 
     \Bigl( \frac{2}{\BP_t(\Ec_t^a) - \BP_t(\bar \Ec_t^\sigma)} + 1 \Bigr) \BE_t[ \gamma_t(G_t)]
    - \eps(t,m) + B \cdot \BP_t(\bar \Ec_t^\sigma)
     \Bigr\} 
\end{aligned}
\end{equation} 

where $\gamma_t(G) = c_t^\mu(G) + c_t^\sigma(G)$.
\end{lemma}

\begin{lemma}[Cumulative Uncertainty Bound]
\label{lemma: upper bound for sum of sigma}
Assume the same as Theorem \ref{theorem: regret upper bound for TS}. Then with probability at least $1-\delta/T$,
\begin{equation}
    \frac{1}{2} \sum_{t=1}^T \min\{ 1, \sigma_t^2 (G_t)\}
    \leq
    \tilde{d} \log(1+\lambda^{-1} T \rho_{\max})
    + 3 C_{\psi} |\mG|^{3/2} \sqrt{T} \lambda^{-1/2} \eps_m 
\end{equation} 

where $\eps_m = o(1)$ as $m \to \infty$ and $C_{\psi}$ is some constant. 
We always have $\tilde{d} \leq |\mG|$.
\end{lemma}

\begin{proof}[Main Proof]
The expected cumulative regret is 
\begin{equation}
    R_T =
    \sum_{t=1}^T \BE[\Delta_t] = \sum_{t=1}^T \BE[\Delta_t \BI_{\mE^{\mu}_t}]
    + \sum_{t=1}^T \BE[\Delta_t \BI_{\bar{\mE}^{\mu}_t}].
\end{equation}
By Lemma~\ref{lemma: high probability bound for event mu}, letting $\BP ( \bar{\mE}_t^{\mu}) \leq \delta/T$ and $\Delta_t \leq B$, we have the upper bound for the second term
\begin{equation}
    \sum_{t=1}^T \BE[\Delta_t \BI_{\bar{\mE}^{\mu}_t}] \leq BT (\delta/T) = B \delta.
\end{equation} \vspace{-2mm}

Now our focus is controlling the first summation term. By Lemma~\ref{lemma: one step regret bound}, almost surely, we have
\begin{equation}
\begin{aligned}
    \BE_t[\Delta_t \ind_{\Ec_t^\mu}]  
    \leq
    \BI_{\mE_t^\mu} \cdot \Bigl\{ 
     \Bigl( \frac{2}{\BP_t(\Ec_t^a) - \BP_t(\bar \Ec_t^\sigma)} + 1 \Bigr) \BE_t[ \gamma_t(G_t)]
    - \eps(t,m) + B \cdot \BP_t(\bar \Ec_t^\sigma)
     \Bigr\} 
\end{aligned}
\end{equation} 
where $\gamma_t(G) = c_t^\mu(G) + c_t^\sigma(G)$. Assuming that $t \geq 5$, we have $t^2 \geq 5 e \sqrt{\pi}$. By Lemma~\ref{lemma: high probability bound for event for exploration}, $ \BP_t(\Ec_t^a) - \BP(\bar \Ec_t^\sigma) \ge \frac1{4 e \sqrt \pi} - \frac{1}{t^2} \ge \frac1{20 e \sqrt \pi}.$
Then, for $t \ge 5$, dropping $\eps(t,m)$ from the bound,
\begin{equation}
    \ex_t[\Delta_t \ind_{\Ec_t^\mu}]  
    \leq 194 \BE_t[ \gamma_t(G_t)] + B t^{-2}
    \le  \bigl( 194  \BE_t[ \min\{1, \gamma_t(G_t)\}] + t^{-2} \bigr) B
\end{equation}
using $40 e \sqrt \pi + 1 \le 194$, $\Delta_t \le B$ and $B \ge 1$. Therefore, we have 
\begin{equation}
\label{eq:half in main proof}
    \sum_{t=1}^T \BE[\BE_t[\Delta_t \ind_{\Ec_t^\mu}]] \le 194 B \sum_{t=5}^T \BE [\BE_t[\min\{1, \gamma_t(G_t)\}]] + 4B + B (\pi^2/6)
\end{equation}
using $\sum_{t=1}^\infty t^{-2} = \pi^2/6$. Note that 
$\gamma_t(G_t) \le \sigma_t(G_t) \sqrt{8 \log(T^2 |\Gc|)} + \eps(T,m)$
for all $t \in [T]$. Then by Cauchy-Schwarz inequality,
\begin{equation}
    \sum_{t=5}^T \min\{1, \gamma_t(G_t)\}
    \le 
    \sqrt{8 T \log(T^2 |\Gc|)} \Bigl(\sum_{t=5}^T \min\{ 1, \sigma_t^2 (G_t)\} \Bigr)^{1/2} + T \eps(T,m) .
\end{equation}
By Lemma~\ref{lemma: upper bound for sum of sigma} and take $m$ sufficiently large such that $3 C_{\psi} |\mG|^{3/2} \sqrt{T} \lambda^{-1/2} \eps_m  \leq \tilde{d} \log(1+\lambda^{-1} T \rho_{\max})$, we have
\begin{equation}
    \sum_{t=1}^T \BE[ \min\{1, \sigma_t^2(G_t) \}] 
    \leq 
    4 \tilde{d} \log(1 + T \rho_{\max} /\lambda) + T (\delta/T).
\end{equation}
Recall that the $\eps(T,m) = C_1 \, T \, m^{-1/6}\sqrt{\log m }$. Take $m$ large enough we have $T \eps(T, m) \le \sqrt{T}$. Then put the above results back into \eqref{eq:half in main proof}, we have:
\begin{equation}
    \sum_{t=1}^T \BE[\BE_t[\Delta_t \ind_{\Ec_t^\mu}]] 
    \leq
    194 B \bigl(
    \sqrt{16 T \log(T |\Gc|)} \cdot
    \sqrt{4 \tilde{d} \log(1 + T \rho_{\max} /\lambda) + \delta} 
    + \sqrt{T} \bigr)
    + 4B + B (\pi^2/6)
\end{equation}
by using $\log(T^2|\Gc|) \le 2 \log (T|\Gc|)$. Therefore, we have our regret bound:
\begin{equation}
\begin{aligned}
    R_T 
    & \leq
    C B \sqrt{\tilde{d} T \log(T |\mG|) \cdot \bigl( 1+\log(1 + T \rho_{\max} /\lambda) \bigr)}
\end{aligned}
\end{equation}
for some universal constant $C$. We have used 
$\deff \ge 1$ and $B \ge 1$, to simplify the bound. Finally, note that $1+\log(1+x) \le 2\log(2+x)$ for all $x \ge 0$. 
\end{proof}

\begin{figure}[t!]
    \centering
    \includegraphics[width=0.94\textwidth, height = 6.2cm]{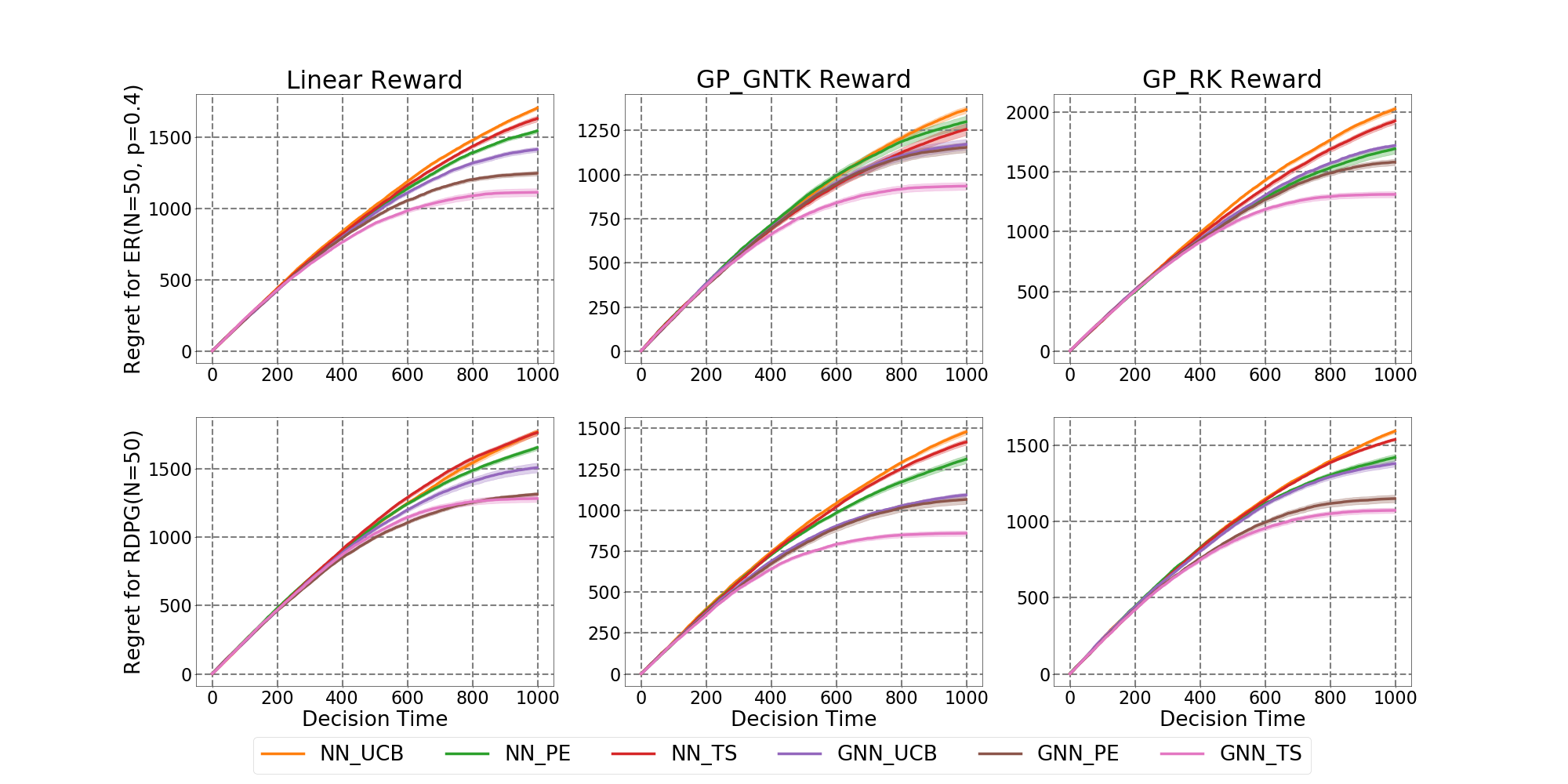}
    \vspace{-1mm}
    \caption{Regret over horizon 
    $T= 1000$ for Erd\"{o}s--R\'{e}nyi random graphs with $p=0.4$ and $N=50$ in the first row and random dot product graphs with $N=50$. Three columns are three types of reward function generation: linear model, Gaussian process with GNTK, Gaussian process with representation kernel. \texttt{GNN-TS} is competitive and robust to different environment settings. } 
    \label{figure: general performance}
\vspace{-1mm}
\end{figure}

\vspace{-1mm}
\section{Experiments}
\label{sec:  experiments}
We create synthetic graph data and generate the rewards through three different mechanisms. For the graph structures, we use random graph models including Erd\"{o}s--R\'{e}nyi and random dot product graph models. The features are generated i.i.d. from the $\mN(0,1)$. The noisy reward is assumed to have $\sigma_{\eps} = 0.01$. Our experiments investigate \texttt{GNN-UCB}, \texttt{GNN-PE}, \texttt{NN-UCB}, \texttt{NN-PE}, and \texttt{NN-TS} as baselines from~\cite{kassraie2022graph}. All performance curves in our empirical studies show an average of over $10$ repetitions with a standard deviation of the corresponding bandit algorithm with horizon $T=1000$. We assume the graph domain is fully observable, $\mG_t = \mG$ for all $t\in[T]$. 
Below is a brief overview of the simulation elements. For more details, see Appendix~\ref{appendix: supplement to experiments}.

\textbf{Random Graph.}  We use two types of random graphs including Erd\"{o}s--R\'{e}nyi (ER) random graphs and random dot product graphs (RDPG).  
ER graphs are generated with edge probability $p$ and number of nodes $N$. RDPGs are generated by modeling the expected edge probabilities as the function of the inner product of features. In the first row of Figure~\ref{figure: general performance}, the
graphs in $\mG$ are from the ER model with $p=0.4$
and in the second row from an RDPG, both of size $N=50$.

\textbf{Reward Function.} To generate the rewards, we use models of three different types: linear model, Gaussian Process (GP) with GNTK, Gaussian process with the representation kernel. 
For the linear model, we have $\mu(G) = \ip{\btheta^{*}, \bar{\bh}^G}$ with true parameter $\btheta^{*} \sim \mathcal N(0, \bI_d)$ and $\bar{\bh}^G = \sum_{i=1}^N \bh^G_i/N$. 
For the GP with GNTK, we fit a GP regression model with empirical GNTK matrix $\bKh \in \BR^{|\mG| \times |\mG|}$ as the covariance matrix of the prior, trained on $\{(G, y_G)\}_{G \in \mG}$ where $\{y_G\}_{G \in \mG}$ are i.i.d. from $\mN(0,1)$.
For the GP with the representation kernel, we trained a GNN for a graph property prediction task and used the mean pooling over all the nodes of the last layer representations as the graph representation, denoted as $\bar\bh^{G}_{\text{rep}}$. 
We then define the representation kernel as $k_{\text{rep}}(G, G') := \ip{\bar\bh^{G}_{\text{rep}}, \bar\bh^{G'}_{\text{rep}}}$ and draw $\mu(\cdot)$ from a zero-mean GP with this covariance function (over $\Gc$).

\textbf{Algorithms.} We investigate two baselines \texttt{GNN-UCB} and \texttt{GNN-PE} along with our proposed \texttt{GNN-TS}. \texttt{GNN-PE} is the proposed state-of-the-art algorithm that selects the graph with the highest uncertainty and eliminates the graph candidates by the upper confidence bounds. All the algorithms in our work use the loss function \eqref{eq: minimize l2 loss} which is different from previous work. All gradients used for our experiments are $\bcg(G;\btheta_{t})$, not $\bcg(G;\btheta_{0})$, unless otherwise specified. 
In addition, in order to show the benefit of considering the graph structure, we include \texttt{NN-UCB}, \texttt{NN-TS}, and \texttt{NN-PE} as our baselines. For these NN-based algorithms, we ignore the adjacency matrix of a graph (setting $\bA = \bI_N$), and pass through the model in \eqref{eq: MLP} and \eqref{eq: GNN} with $\bh^G_i = \bX_{i*}$. 
The MLPs in our experiments have   $L=2$ layers and width $m=512$. We use SGD as the optimizer, with mini-batch size $5$, and train for $30$ epochs. For the tuning of the hyperparameters $(\eta,\lambda)$ and other algorithmic setup, see Appendix~\ref{appendix: supplement to experiments}.
The matrix inversion in the algorithms is approximated by diagonal inversion across all policy algorithms.

\textbf{Regret Experiments.} 
In Figure~\ref{figure: general performance}, we show the performance of all the algorithms for the six possible environments: ER or RDPG model coupled with either of the three reward models. We set the size of the graph domain to $|\mG|=100$ in Figure~\ref{figure: general performance} and we experiment across different $|\mG|$ in Appendix~\ref{appendix: supplement to experiments}. 
Figure~\ref{figure: general performance} demonstrates that \texttt{GNN-TS} consistently outperforms the baseline algorithms and is robust to all types of random graph models and reward function generations in our experiment. In addition, GNN-based algorithms are clearly better than NN-based algorithms in graph action bandit settings. 



\bibliography{reference}

\newpage


\appendix

\newcommand\Ac{\mathcal A}
\newcommand\xib{\bar \xi}
\newcommand\Ub{\bar U}
\newcommand\Vb{\bar V}
\newcommand\Wb{\bar W}
\newcommand\psib{\bar \psi}
\newcommand\Psib{\bar \Psi}
\newcommand\rhobmax{\bar \rho_{\max}}
\newcommand\Ab{\bar A}
\newcommand\bAb{\boldsymbol{\bar A}}
\newcommand\ab{\bar a}
\newcommand\tht{\widetilde{\btheta}}

\section{Proof for Lemmas in Regret Analysis}
\label{appendix: proof for lemmas in regret analysis}

\subsection{Notations}
\label{appendix: notations}
In the following parts, we further define the some notations to represent the linear and kernelized models:
\begin{equation}
\begin{aligned}
    \bG_t &= [\bcg(G_1;\btheta_0),...,\bcg(G_t;\btheta_{t-1}))] \in \BR^{p \times t} \\
    \bar{\bG}_t &=[\bcg(G_1;\btheta_0),...,\bcg(G_t;\btheta_0)] \in \BR^{p \times t} \\
    \bmu_{t} &= [\mu(G_1),...,\mu(G_t)]^{\top} \in \BR^{t \times 1} \\
    \by_t &= [y_1, ..., y_t]^{\top} \in \BR^{t \times 1}\\
    \bepsilon_{t} &= [\eps_1,...,\eps_t]^{\top} \in \BR^{t \times 1}.
\end{aligned}
\end{equation}
Then we define the uncertainty estimate with initial gradient $\btheta_0$:
\begin{equation}
    \bar{\sigma}_{t}^2(G) 
    = \frac{1}{m} \norm{\bcg(G;\btheta_0)}_{\bar{\bU}_{t-1}^{-1}}^2
    \quad \text{and} \quad 
    \bar{\bU}_t 
    = \lambda \bI_{p} + \sum_{i=1}^t \bcg(G_i;\btheta_0) \bcg(G_i;\btheta_0)^{\top}/m .
\end{equation}

\subsection{Proof of Lemma~\ref{lemma: high probability bound for event mu}}
\label{appendix: proof of lemma high probability bound for event mu}




 Let us write 
\[
\tht_{t-1} := \bar{\bU}_{t-1}^{-1} \bar{\bG}_{t-1} \by_{t-1}/m
\] 
for the ridge regression solution. We will need the following auxiliary lemmas:


\begin{lemma} [Taylor Approximation of a GNN]
\label{lemma: taylor approximation of gnn}
Suppose learning rate $\eta \leq (\tilde{C} mL+m\lambda)^{-1}$ for some constant $\tilde{C}$, then for any fixed $t \in [T]$ and $G \in \mG$, with probability at least $1-\delta$
\begin{equation}
\begin{aligned}
    |\fgnn(G; \btheta_t^{(J)}) - \fgnn(G; \btheta_0) - \langle \bcg(G; \btheta_0), \btheta_t^{(J)} - \btheta_0 \rangle | 
    \leq C L^3 \Bigl(\frac{R^2 + \sigma_{\eps}^2}{m \lambda}\Bigr)^{2/3} \sqrt{m \log(m)}
\end{aligned}
\end{equation}
where $C$ is some constant independent of $m$ and $t$.
\end{lemma}

\begin{lemma}
\label{lemma: bound for gradient error inner product}
Suppose $m \ge poly(R, \sigma_{\eps}, L, \lambda^{-1}, |\mG|, \rho_{min}^{-1}, \log(L N |\mG|/\delta))$ given a fixed $\delta \in (0,1)$ and learning rate $\eta \leq (\tilde{C} mL+m\lambda)^{-1}$ for some constant $\tilde{C}$. For $G \in \mG_t$ and $t > 1$, with probability at lease $1-\delta$, 
\begin{equation}
\begin{aligned}
    |\langle \bcg(G;\btheta_0), \btheta_{t-1}-\btheta_0 
    -\tht_{t-1}
    \rangle|
    \leq 
    C
    \bar{\sigma}_{t}(G)
\end{aligned}
\end{equation}
where $C = (C_1 (2-\eta m \lambda)^J + C_2 ) \sqrt{\frac{\sigma_{\eps}^2 + R^2}{\lambda } (1+\frac{3 \rho_{\max} }{2 \lambda})} $ with $C_1 = \mO(1)$ and $C_2 = \mO(\lambda^{1/3})$.
\end{lemma}

\begin{lemma}
\label{lem:mu:is:linear}
Fix $\delta \in (0,1)$ and let
$
m = \Omega (L^{10} T^4 |\mG|^6 \rho_{min}^{-4} \log(L N^2 |\mG|^2/ \delta) ).
$
Then, there exists $\btheta^* \in \reals^p$ with $\sqrt{m} \norm{\btheta^{*}}_2 \leq \sqrt{2} R$ such that 
with probability at least $1-\delta$, 
\begin{equation}
\begin{aligned}
    \mu(G) &= \langle \bcg(G;\btheta_0), \btheta^{*} \rangle, \quad \text{for all $G \in \Gc$}\\
    \log\det( \lambda^{-1} \bar{\bU}_t)
    & \leq 
    \log\det( \bI_{|\mG|} + \lambda^{-1} t \bK) + 1.
\end{aligned}
\end{equation}
\end{lemma}

\begin{lemma}
\label{lemma: high probability bound for sigma difference}
With probability at least $1-\delta$, we have 
\begin{equation}
    |\bar{\sigma}_{t}(G) - \sigma_{t}(G)| 
    \leq 
    C t
    \lambda^{-1/6} L^{9/2} (R^2 + \sigma_{\eps}^2)^{1/6} m^{-1/6} \sqrt{\log(m)}.
\end{equation}
\end{lemma}


\medskip
We choose an arbitrary small $\delta \in (0,1)$ and set $\delta_i = \delta/(5T)$ for $i=1,\dots,5$.
For all $\forall G \in \mG_t$, we have
\begin{equation}
\begin{aligned}
    |\fgnn(G;\btheta_{t-1}) - \mu(G)|
    \leq
    \underbrace{|\fgnn(G;\btheta_{t-1}) -
    \langle \bcg(G;\btheta_0), 
    \tht_{t-1}
    \rangle|}_{:=I_1}
    + 
    \underbrace{|\mu(G) - \langle \bcg(G;\btheta_0), \tht_{t-1}\rangle|}_{:=I_2}.
\end{aligned}
\end{equation}
We then turn to bounding $I_1$ and $I_2$.  Throughout the proof, let
\begin{equation}
    \gamma_m  := m^{-1/6}\sqrt{\log m }
\end{equation}
\textbf{Bounding $I_1$}:
By Lemma~\ref{lemma: taylor approximation of gnn} and Lemma~\ref{lemma: bound for gradient error inner product}, with probability at least $1-\delta_1-\delta_2$, 
\begin{equation}
\begin{aligned}
    I_1
    &=  
    |\fgnn(G;\btheta_{t-1}) -
    \langle \bcg(G;\btheta_0), 
    \tht_{t-1}
    \rangle| \\
    &\leq 
    |\fgnn(G;\btheta_{t-1}) - \langle \bcg(G;\btheta_0), \btheta_{t-1}-\btheta_0 \rangle| +
    | \langle \bcg(G;\btheta_0), \btheta_{t-1}-\btheta_0 -\tht_{t-1} \rangle| \\
    &\leq C_0 L^3 \gamma_m  +  \tilde{C}_2 \, \bar{\sigma}_{t}(G).
\end{aligned}
\end{equation}
where $C_0:=\tilde{C}_1 \Bigl(\frac{R^2 + \sigma_{\eps}^2}{ \lambda}\Bigr)^{2/3}$ and $\tilde{C}_2:=(\bar{C}_1 (2-\eta m \lambda)^J + \bar{C}_2 \lambda^{1/3}) \sqrt{\frac{\sigma_{\eps}^2 + R^2}{\lambda } (1+\frac{3 \rho_{\max} }{2 \lambda})}$ for some constant $\bar{C}_1$, $\bar{C}_2$. 
For $\lambda \gtrsim (\sigma_\eps^2 + R^2)^3 + \rhomax$, we have $C_0, \tilde C_2 \lesssim 1$ subject to the constraint in $\eta$ in Lemma~\ref{lemma: bound for gradient error inner product}. Thus, we obtain
\[
I_1 \lesssim L^3 \gamma_m + \bar \sigma_t(G).
\]

\textbf{Bounding $I_2$}: 
By Lemma~\ref{lemma: approximation from gntk}, with at least probability $1-\delta_3$,
for all $G \in \Gc$, we have 
\[
 I_2 = |\ip{\bcg(G;\btheta_0),  \btheta^{*} - \tht_{t-1}}|.
\]
Recall that $\by_{t-1} = \bmu_{t-1} + \bepsilon_{t-1}$ and by Lemma~\ref{lem:mu:is:linear}, we have $\bmu_{t-1} =  \bar{\bG}_{t-1}^\top \btheta^*$. Then,
\[
\tht_{t-1}   = \bar{\bU}_{t-1}^{-1} \bar{\bG}_{t-1} \bar{\bG}_{t-1}^\top \btheta^* /m +
\bar{\bU}_{t-1}^{-1}\bar{\bG}_{t-1} \bepsilon_{t-1} / m
\]

We have $\bar \bU_{t} = \lambda \bI_p + \bar \bG_t \bar \bG_t^{\top} / m$. Hence, $\bar \bU_t^{-1} \bar \bG_t \bar \bG_t^{\top} / m = \bar \bU_t^{-1} (\bar \bU_{t} - \lambda \bI_p) = \bI_p - \lambda \bar \bU_t^{-1}$. This gives
\[
\tht_{t-1}   = \btheta^* - \lambda \bar \bU_{t-1}^{-1}\btheta^* +  \frac1{\sqrt m} \,\bar{\bU}_{t-1}^{-1} \bS_{t-1}
\]
where we have defined $\bS_{t-1} := \frac1{\sqrt m}\bar{\bG}_{t-1} \bepsilon_{t-1}$. Thus, we have 
\begin{align}\label{eq:I2:decomp}
I_2 \le \lambda |\ip{\bcg(G;\btheta_0), \bar \bU_t^{-1} \btheta^*}| + \frac1{\sqrt m} |\ip{\bcg(G;\btheta_0), \bar \bU_{t-1}^{-1} \bS_{t-1}}|
\end{align}
Recall that $\sqrt m \bar{\sigma}_{t}(G) 
    =  \norm{\bcg(G;\btheta_0)}_{\bar{\bU}_{t-1}^{-1}}$. Since $\bar{\bU}_{t-1}^{-1} \preccurlyeq \frac{1}{\lambda} \bI_p$, for any vector $\bv$, we have $\norm{\bv}_{\bar{\bU}_{t-1}^{-1}} \le \frac1{\sqrt{\lambda}} \norm{\bv}$.
Then, for the first term in~\eqref{eq:I2:decomp}, we have 
\begin{align*}
\lambda |\bcg(G;\btheta_0)^{\top} \bar{\bU}_{t-1}^{-1} \btheta^{*} | 
    &\leq 
    \lambda \norm{\bcg(G;\btheta_0)}_{\bar{\bU}_{t-1}^{-1}} \cdot \norm{\btheta^{*}}_{\bar{\bU}_{t-1}^{-1}} \\
    &\leq  \sqrt{m} \bar{\sigma}_{t}(G) \cdot \sqrt{\lambda} \norm{\btheta^{*}}_2 
    \leq  \bar{\sigma}_{t}(G) \sqrt{2 \lambda} R    
\end{align*}
where we have used Cauchy-Schwarz inequality for $\ip{\cdot,\cdot}_{\bU_{t-1}^{-1}}$ and Lemma~\ref{lem:mu:is:linear}. For the second term in~\eqref{eq:I2:decomp}, we have
\begin{align*}
    \frac{1}{\sqrt m}|  \bcg(G;\btheta_0)^{\top} \bar{\bU}_{t-1}^{-1}\bS_{t-1} | 
    & \leq 
    \frac{1}{\sqrt m} \norm{\bcg(G;\btheta_0)}_{\bar{\bU}_{t-1}^{-1}}
    \norm{\bS_{t-1}}_{\bar{\bU}_{t-1}^{-1}}
    =
    \bar{\sigma}_{t}(G)\cdot 
     \norm{\bS_{t-1}}_{\bar{\bU}_{t-1}^{-1}}
\end{align*}
By Theorem 20.4 of~\cite{lattimore2020bandit}, 
with probability at least $1-\delta_4$, we have 
\[
 \frac1{\sigma_\eps^2}\norm{\bS_{t}}_{\bar{\bU}_{t}^{-1}}^2 
 \le 2 \log(1/\delta_4) + \log \deg (\lambda^{-1} \bar \bU_t), \quad \text{for all}\; t \in \nats.
\]


By Lemma~\ref{lem:mu:is:linear}, with high probability, 
\begin{equation}
    \log \det (\lambda^{-1} \bar \bU_t)
    \le  \log \det( \bI_{|\Gc|} + T \bK/\lambda) + 1 
    \le 2 \deff  \log(1 + T \rho_{\max} /\lambda).
\end{equation}
Using $\lambda \gtrsim \rho_{\max}$, we have $\log \det(\lambda^{-1} \bar \bU_t) \lesssim \deff \log(T) + 1 \lesssim \deff \log(T)$.
We also have $\log(1/\delta_4) = \log (5T) \lesssim \log(T) \lesssim \deff \log (T)$.

Putting the pieces together, we have 
\[
 \frac{1}{\sqrt m} |  \bcg(G;\btheta_0)^{\top} \bar{\bU}_{t-1}^{-1}\bS_{t-1}  | \lesssim \sigma_\eps \sqrt{\deff \log T} \cdot \bar \sigma_t(G).
\]
Combining with the first term, we obtain
\[
I_2 \lesssim \bigl(  \sigma_\eps \sqrt{\deff \log T}  + \sqrt \lambda  R \bigr) \bar \sigma_t(G).
\]
Combining with the bound on $I_1$, we have
\begin{align*}
 |\fgnn(G;\btheta_{t-1}) - \mu(G)| &\lesssim  L^3 \gamma_m + 
 \bigl( 1 + \sigma_\eps \sqrt{\deff \log T}  + \sqrt \lambda  R \bigr) \,\bar \sigma_t(G) \\
 &=: L^3 \gamma_m + \alpha \, \bar \sigma_t(G)
\end{align*}
where we have set $\alpha :=  1 + \sigma_\eps \sqrt{\deff \log T}  + \sqrt \lambda  R$ for simplificty.

By Lemma~\ref{lemma: high probability bound for sigma difference}, with probability at least  $1-\delta_5$,
\[
\bar{\sigma}_{t}(G) -  \sigma_{t}(G) \le  
    C t
     L^{9/2} \Bigl(\frac{R^2 + \sigma_{\eps}^2}\lambda \Bigr)^{1/6} \gamma_m  \lesssim t \cdot L^{9/2} \gamma_m
\]
using the assumption $\lambda \gtrsim R^2 + \sigma_\eps^2$. We obtain
\begin{align*}
 |\fgnn(G;\btheta_{t-1}) - \mu(G)| &\lesssim  L^3 \gamma_m + t \cdot \alpha L^{9/2} \gamma_m  +
 \alpha  \sigma_t(G) \\
 &\le 2 t \cdot \alpha L^{9/2} \gamma_m  +
 \alpha  \sigma_t(G)    
\end{align*}
since $t \ge 1$ and $\alpha \ge 1$. Taking $\nu \ge \alpha$ finishes the proof.

\subsection{Proof of Lemma~\ref{lemma: high probability bound for event for exploration}}
\label{appendix: proof of lemma high probability bound for event for exploration}
\begin{proof} [Proof of Lemma~\ref{lemma: high probability bound for event for exploration}]
Conditioned on $\Fc_t$, we have 
\begin{equation}
    \rh_t(G) \given \Fc_t
    \sim \mN(\fgnn(G;\btheta_{t-1}), \nu^2 \sigma_t^2(G)).
\end{equation}
Using standard Gaussian tail bound, followed by a union bound gives
\begin{equation}
    \pr_t(|\rh_t(G) - \fgnn(G;\btheta_{t-1})| \ge  \nu \sigma_t(G) \cdot u) \le |\Gc_t| e^{-u^2/2}
\end{equation}
which shows the first assertion by letting $u=\sqrt{2\log(t^2 |\mG_t|)}$. 

For the second assertion, it is enough to note that $\pr(Z \ge 1) \ge (4e\sqrt \pi)^{-1}$ for $Z \sim \mN(0,1)$.
\end{proof}

\subsection{Proof of Lemma~\ref{lemma: one step regret bound}}
\label{appendix: proof of lemma one step regret bound}
\begin{proof}[Proof of Lemma~\ref{lemma: one step regret bound}]
Our proof is inspired from the proof in~\cite{wu2022residual}.
Recall that $c_t^\mu(G) = \nu \sigma_t(G) + \eps(t,m)$ and  $c_t^\sigma(G) :=  \nu \sigma_{t}(G) \sqrt{ 2\log (t^2 |\mG_t|)}$ and 
\begin{equation}
\begin{aligned}
     \mE_t^{\mu} 
     &= \{\forall G \in \mG_t, |\fgnn(G;\btheta_{t-1}) - \mu(G)| \leq c_{t}^{\mu}(G)\} \\
     \mE_t^{\sigma} &= \{\forall G \in \mG_t, |\rh_{t}(G) - \fgnn(G;\btheta_{t-1}) | \leq c_{t}^{\sigma}(G)\}
\end{aligned}
\end{equation}
Let
$\gamma_t(G) = c_t^\mu(G) + c_t^\sigma(G)$
and $c_t(G) = \gamma_t(G) + \eps(t, m)$.    
Then, on $\Ec_t^\mu \cap \Ec_t^\sigma$, by triangle inequality,
\begin{align}\label{eq:rh:mu:dev}
|\rh_t(G) - \mu(G)| \le \gamma_t(G).
\end{align}
We also recall that $\Ec^a_t := \{\rh_{t}(G^{*}_t)-\fgnn(G^{*}_t;\btheta_{t-1}) 
> \nu \sigma_t(G^*_t) \}$. Then, on  $\Ec_t^\mu \cap \Ec_t^a$,  we have 
\begin{align}
\rh_{t}(G^{*}_t)
&> \fgnn(G^{*}_t;\btheta_{t-1})  + \nu \sigma_t(G^*_t) \notag \\
&\ge \mu(G^*_t) - c_t^\mu(G^*_t) + \nu \sigma_t(G^*_t) \notag \\
&= \mu(G^*_t) - \eps(t,m) \label{eq:rh:Gs:lower:bound}
\end{align}

Recall that $\Delta_t := \mu(G^*_t) - \mu(G_t)$ for convenience. Consider the set of unsaturated actions 
\[
\Uc_t = \bigr\{ G \in \Gc_t: \; 
\mu(G^*_t) < \mu(G) + c_t(G) \bigl\}
\]
and let $\bar G_t$ be the least uncertain 
unsaturated action at time $t$:
\begin{equation}
    \bar{G}_t := \argmin_{G \in \Uc_t} c_t(G).
\end{equation}
By $\Gb_t \in \Uc_t$, we have   
$
\Delta_t \le  c_t(\Gb_t) +  \mu(\Gb_t) - \mu(G_t).
$
Applying~\eqref{eq:rh:mu:dev}, twice, on $\Ec_t^\mu \cap \Ec_t^\sigma$, we have 
\begin{align*}
\Delta_t &\le c_t(\Gb_t) + \gamma_t(\Gb_t) + \gamma_t(G_t) + \rh_t(\Gb_t) - \rh_t(G_t) \\
&\le \ c_t(\Gb) + \gamma_t(\Gb_t) + \gamma_t(G_t) 
\end{align*}
for all $G \in \Gc_t$ 
where the second inequality follows since $G_t$ maximizes $\rh_t(\cdot)$ over $\Gc_t$, by design.

Recall that $\ex_t[\cdot] = \ex[\cdot \given \Fc_t]$, where $\Fc_t$ is the history up to (but not including) time $t$. Given $\Fc_t$, the event $\Ec_t^\mu$ is deterministic while $\Ec_t^\sigma$ is only random due to the independent randomness in the sampling step~\eqref{eq:TS:sampling}. Next, we have 
\begin{align}
    \ex_t[\Delta_t \ind_{\Ec_t^\mu}]  &= 
    \ind_{\Ec_t^\mu} \cdot \ex_t[\Delta_t] \notag \\
    &= \ind_{\Ec_t^\mu} \cdot \bigl( \ex_t[\Delta_t \ind_{\Ec_t^\sigma}] +
     \ex_t[\Delta_t \ind_{\bar \Ec_t^\sigma}] \bigr) \notag \\
     &\le 
     \ind_{\Ec_t^\mu} \cdot \bigl( \ex_t[\Delta_t \ind_{\Ec_t^\sigma}] + B \, \pr_t(\bar \Ec_t^\sigma)\label{eq:Delta:decomp:1}
     \bigr)
\end{align}
using the boundedness Assumption~\ref{assumption: bounded reward differences}. Here, we are using the fact that $\Ec_t^\mu$ is measurable w.r.t. $\Fc_t$, hence it is deterministic conditioned on $\Fc_t$. Due to factor $\ind_{\Ec_t^\mu}$ in the above, the bound is trivial when $\Ec_t^\mu$ fails, so for the rest of the proof we assume that $\Ec_t^\mu$ holds (conditioned on $\Fc_t$). 

We have
\begin{align*}
 \ex_t[\Delta_t \ind_{ \Ec_t^\sigma}] &\le c_t(\Gb_t) + \gamma_t(\Gb_t) + \ex_t[\gamma_t(G_t) \ind_{  \Ec_t^\sigma}]    \\
 &\le 2 c_t(\bar G_t) - \eps(t,m) + \ex_t[\gamma_t(G_t)]
\end{align*}
where we have used the definition of $c_t(\cdot)$ and dropped the indicator $\ind_{\Ec_t^\sigma}$ to get a further upper bound. It remains to bound $c_t(\bar G_t)$ in terms of $\gamma_t(G_t)$.

Since $\bar G_t$ is the least uncertain unsaturated action, we have 
\[
c_t(\bar G_t) \ind\{ G_t \in \Uc_t\} \le c_t(G_t).
\]
Multiplying both sides by $\ind_{\Ec_\sigma^t}$, taking $\ex_t[\cdot]$, and rearranging 
\[
c_t(\bar G_t) \le \frac{\ex_t[ c_t(G_t) \ind_{\Ec_\sigma^t}]}{\pr_t(\{G_t \in \Uc_t\} \cap \Ec_t^\sigma)} \le 
\frac{\ex_t[ \gamma_t(G_t)]}{\pr_t(\{G_t \in \Uc_t\} \cap \Ec_t^\sigma)}.
\]
It remains to bound the denominator.

Recall that $G_t$ maximizes $\rh_t(\cdot)$ over the entire $\Gc_t$. Thus, if
\begin{align}\label{eq:key:1}
\rh_t(G_t^*) > \max_{G \, \in\,  \bar \Uc_t} \rh_t(G)    
\end{align}
then $G_t$ cannot belong to $\bar\Uc_t$, hence $G_t \in \Uc_t$. On $\Ec_t^\mu \cap \Ec_t^\sigma$, for any $G \in \bar \Uc_t$, we have 
\begin{align*}
\rh_t(G) \le \mu(G) + \gamma_t(G) &\le \mu(G^*_t) -c_t(G) + \gamma_t(G) \\&\le \mu(G^*_t) - \eps(t,m)
\end{align*}
where the second inequality is by the definition of $\bar \Uc_t$. Then for~\eqref{eq:key:1} to hold on $\Ec_t^\mu \cap \Ec_t^\sigma$, it is enough to have $\rh_t(G_t^*) > \mu(G_t^*) - \eps(t,m)$. But this holds on~$\Ec_t^\mu \cap \Ec_t^a$ by~\eqref{eq:rh:Gs:lower:bound}. That is,
\begin{align*}
    \Ec_t^a  \cap \Ec_t^\mu \cap \Ec_t^\sigma &\subset \{\rh_t(G_t^*) > \mu(G_t^*) - \eps(t,m)\} \cap \Ec_t^\mu \cap \Ec_t^\sigma \\
    &\subset 
    \{ \rh_t(G_t^*) > \max_{G \, \in\,  \bar \Uc_t} \rh_t(G)
    \} \cap \Ec_t^\mu \cap \Ec_t^\sigma \\
    &\subset  \{G_t \in \Uc_t\} \cap \Ec_t^\mu \cap \Ec_t^\sigma.
\end{align*}
Assuming as before  that $\Ec_t^\mu$ holds, we have
\begin{align}
    \pr_t(\Ec_t^a \cap \Ec_t^\sigma) \le \pr_t(\{G_t \in \Uc_t\} \cap \Ec_t^\sigma).
\end{align}
We have $\pr_t(\Ec_t^a \cap \Ec_t^\sigma) \ge \pr_t(\Ec_t^a) - \pr_t(\bar \Ec_t^\sigma)$. Putting the pieces together
\[
c_t(\bar G_t) \le \frac{\ex_t[ \gamma_t(G_t)]}{\pr_t(\Ec_t^a) - \pr_t(\bar \Ec_t^\sigma)}
\]
and we obtain
\[
\ex_t[\Delta_t \ind_{ \Ec_t^\sigma}] \le 
\Bigl( \frac{2}{\pr_t(\Ec_t^a) - \pr_t(\bar \Ec_t^\sigma)} + 1 \Bigr) \ex_t[ \gamma_t(G_t)]
    - \eps(t,m)
\]

Combining with~\eqref{eq:Delta:decomp:1} the result follows.
\end{proof}

\subsection{Proof of Lemma~\ref{lemma: upper bound for sum of sigma}}
\label{appendix: proof of lemma upper bound for sum of sigma}

\begin{proof}[Proof of Lemma~\ref{lemma: upper bound for sum of sigma}]
For simplicity, we define
\begin{align}
    \bcg_t := \frac1{\sqrt{m}} \bcg(G_t; \btheta_{t-1}), \quad 
    \bar{\bcg}_t :=  \frac1{\sqrt{m}} \bcg(G_t; \btheta_0).
\end{align}
Then, recall that
\[
\sigma_t^2(G_t) = \norm{\bcg_t}_{\bU_{t-1}^{-1}}^2,
\quad 
\bU_{t-1} = \lambda \bI_p +  \sum_{i=1}^{t-1} \bcg_t \bcg_t^{\top}.
\]
Note that $\bU_t = \bU_{t-1} + \bcg_t \bcg_t^{\top}$. 

Then we introduce following Lemmas: 
\begin{lemma}[Elliptical Potential]
\label{lem:elliptical:pot}
Assume that $\bU_t = \bU_{t-1} + \bcg_t \bcg_t^{\top}$ for all $t \in [T]$. Then,
\begin{equation}
    \sum_{t=1}^T 
    \min\{1 , \norm{\bcg_t}_{\bU_{t-1}^{-1}}^2\} 
    \leq
    2 \log \Big( \frac{\det \bU_T}{ \det \bU_0} \Big).
\end{equation}
\end{lemma}

\begin{lemma}
\label{lem:logdet:pert}
Let $\bA = [\ba_1\; \ba_2\; \cdots \;\ba_n]$ and $\bAb = [\bar{\ba}_1\; \bar{\ba}_2\; \cdots \;\bar{\ba}_n]$ be $p \times n$ matrices, with columns $\{\ba_i\}$ and $\{\bar{\ba}_i\}$, respectively. Assume that for $\eps \le C$, we have 
\[
\norm{\ba_i - \bar{\ba}_i} \le \eps, \quad \norm{\ba_i} \le C
\]
for all $i$. Then,
\begin{align}
\log \det(\bI_p + \bA\bA^{\top}) 
&\le \log \det (\bI_p + \bAb \bAb^{\top}) + p \log(1+3 C n \eps) \\
\log \det(\bI_p + \bA\bA^{\top}) &\le \log \det (\bI_n + \bAb^{\top} \bAb) + 3 C n^{3/2} \eps.
\end{align}
\end{lemma}

By Lemma~\ref{lem:elliptical:pot}, we have
\begin{equation}
    \frac{1}{2} \sum_{t=1}^T \min\{1 , \sigma_t^2 (G_t)\} 
    \leq
    \log \Big( \frac{\det \bU_T}{ \det \bU_0} \Big)
    = \log \det (\lambda^{-1} \bU_T) =: \log \det(\bV_T)
\end{equation}
using $\det(\bU_0) = \det(\lambda \bI_p) = \lambda^p$, and defining $\bV_t:= \lambda^{-1} \bU_t$.

Let $\mG = \{G^j: j \in[|\mG|]\}$ be the collection of all the graphs and $n_j(t)$ be the number of graphs which are equal to $G^j \in \Gc$ in our selection of graphs up to and including time $t$, i.e $n_j(t):=\sum_{t=1}^t \BI_{G_i=G^j}$. 
Let
\begin{align}
    \bpsi_j := \frac{1}{\sqrt{m}} \bcg(G^j; \btheta_{t-1}), 
    \quad 
    \bar{\bpsi}_j :=  \frac{1}{\sqrt{m}} \bcg(G^j; \btheta_0)
\end{align}
and let $\bPsi$ and $\bar \bPsi$ be the corresponding $p \times |\mG|$ matrices with the above columns.
Then, we have
\begin{equation}
    \sum_{i=1}^T \bcg_i  \bcg_i^{\top} 
    = \sum_{j=1}^{|\mG|} n_j(T) \bpsi_j \bpsi_j^{\top}
    = \bPsi \bD \bPsi^{\top} \preceq T \cdot \bPsi \bPsi^{\top}
\end{equation}
where $\bD \in \BR^{|\mG| \times |\mG|}$ is the diagonal matrix with diagonal elements $\{n_j(T)\}_{j=1}^{|\mG|}$ and the last inequality due to $n_j(T) \le T$ for all $j \in [|\mG|]$. 

Note that $\bV_T = \bI_p + \lambda^{-1} \sum_{i=1}^T \bcg_i \bcg_i^{\top}$, hence
\begin{equation}
    \log \det(\bV_T) \leq \log \det(\bI_p + \lambda^{-1} T \cdot \bPsi \bPsi^{\top} ).
\end{equation}

By Lemma \ref{lemma: bound for gradient norm}, fix a $\delta_1 \in (0,1)$, we have the following bound for $\norm{\bpsi_j}_2$ and $\norm{\bpsi_j - \bar{\bpsi}_j}_2$, with probability at least $1-\delta_1$, 
\begin{equation}
\begin{aligned}
    \norm{\bpsi_j}_2 
    & \leq 
    \frac{1}{N} \sum_{i \in \mV(G^j)} \norm{\gmlp(\bh_i^{G^j};\btheta_{t-1})/\sqrt{m} }_2 
    \leq C_{\psi} \\
    \norm{\bpsi_j - \bar{\bpsi}_j}_2
    & \leq 
    \frac{1}{N} \sum_{i \in \mV(G^j)} \norm{\gmlp(\bh_i^{G^j};\btheta_{t-1})/\sqrt{m} -  \gmlp(\bh_i^{G^j};\btheta_{0})/\sqrt{m}}_2
    \leq \eps_m
\end{aligned} 
\end{equation}
where $\eps_m = o(1)$ as $m \to \infty$ and $C_\psi$ is $C_7 \sqrt{L}$ in Lemma \ref{lemma: bound for gradient norm}. 

Then, applying Lemma~\ref{lem:logdet:pert} with $n = |\mG|$, $\bA = \sqrt{\lambda^{-1} T} \bPsi$, $\bAb = \sqrt{\lambda^{-1} T} \bar{\bPsi}$ and  $\eps$ replaced with $\sqrt{\lambda^{-1} T} \eps_m$, we obtain
\begin{align}\label{eq:key:log:det:ineq}
\log \det(\bV_T) & \le \log \det(\bI_{|\Gc|} + \lambda^{-1} T 
\cdot \bar\bPsi^{\top} \bar\bPsi
) + 3 C_\psi |\Gc|^{3/2} \sqrt{T} \lambda^{-1/2} \eps_m 
\end{align}
Recall $ \bKh = \bar\bPsi^{\top} \bar\bPsi$ and $\hat{\rho}_{\max} = \lambda_{\max}(\bKh)$ and note that $\bKh$ is the finite-width GNTK matrix. By Lemma~\ref{lemma: approx GNTK rho max}, with high probability,
$\hat{\rho}_{\max} \leq \rho_{\max} + \eps_{\rho, m}$ and note that $\eps_{\rho, m} = \Omega(m^{-1/4})$. Dropping $\eps_{\rho, m}$ by large enough $m$, we have
\begin{align*}
    \log \det(\bI_{|\Gc|} + 
    \lambda^{-1} T \cdot {\bar\bPsi}^{\top} \bar{\bPsi}) \leq
    |\mG| \log(1 + T \rho_{\max}/\lambda).
\end{align*}
Putting the pieces together with the definition of effective dimension $\tilde{d}$ in \eqref{eq:def of effective dimension} finishes the proof.
\end{proof}

\subsection{Proof of Lemma \ref{lem:elliptical:pot}}

\begin{proof}[Proof of Lemma~\ref{lem:elliptical:pot}]
    Since $\min\{1, x\} \leq 2 \log (1+x)$ for $x \ge 0$, we have
\begin{align*}
    \sum_{t=1}^T 
    \min\{1 , \norm{\bcg_t}_{\bU_{t-1}^{-1}}^2\} 
    &\le 2 \sum_t 
    \log(1+\norm{\bcg_t}_{\bU_{t-1}^{-1}}^2) \\
    &= 2 \sum_{t=1}^T \log \Bigl( \frac{\det \bU_t}{\det \bU_{t-1}}  \Bigr)
    = 2 \log \Big( \frac{\det \bU_T}{ \det \bU_0} \Big)
\end{align*}
where the first equality follows from $\det(\bA + \bv\bv^{\top}) = \det(\bA)(1 + \bv^{\top} \bA^{-1} \bv)$, obtained by an application of Sylvester's determinant identity: $\det(\bI + \bA\bB) = \det(\bI + \bB\bA)$. 
\end{proof}

\subsection{Proof of Lemma \ref{lem:logdet:pert}}

\begin{proof}[Proof of Lemma~\ref{lem:logdet:pert}]
Note that 
\begin{align*}
\opnorm{\ba_i \ba_i^{\top} - \bar{\ba}_i \bar{\ba}_i^{\top}} 
&= \opnorm{\ba_i (\ba_i - \bar{\ba}_i)^{\top}  - (\bar{\ba}_i - \ba_i) \bar{\ba}_i^{\top}}  \\
&\le (\norm{\ba_i}  + \norm{\bar{\ba}_i}) \norm{\ba_i - \ba_i} \le (2C + \eps) \eps \le 3 C\eps
\end{align*}
Let $\bV = \bI_p + \bA\bA^{\top}$ and $\bar{\bV} = \bI_p + \bAb \bAb^{\top}$. We have
\[
\opnorm{\bV - \bar{\bV}} \le \sum_{i=1}^n \opnorm{\ba_i \ba_i^{\top} - \bar{\ba}_i \bar{\ba}_i^{\top}}  \le n  \cdot 3 C \eps
\]
Write $\lambda_i(\bV)$ for the $i$th eigenvalue of matrix $\bV$. By  Weyl's inequality
$|\lambda_i(\bV) - \lambda_i( \bar{\bV} )| \le  3 C n \eps$.
Then, 
\begin{align*}
    \log \det (\bV) = \sum_{i=1}^p \log \lambda_i(\bV) 
    &\le \sum_i \log\bigl(\lambda_i(\bar{\bV})  + 3 C n \eps \bigr) 
    \\
    &= \sum_i \log(\lambda_i(\bar{\bV}) ) +\sum_i \log\Bigl(1  + \frac{3 C n \eps }{\lambda_i(\bar{\bV})}\Bigr) \\
    &\le \log \det(\bar{\bV}) + p \log(1+ 3 C n \eps)
\end{align*}
using $\lambda_i(\bar{\bV}) \ge 1$. This proves one of the bounds. 

For the second bound, let $\bW = \bI_n + \bA^{\top} \bA$ and $\bar{\bW} = \bI_n + \bAb^{\top} \bAb$. Then, then by concavity of the $\bX \mapsto \log \det(\bX)$ and the fact that its derivative is $\bX^{-1}$ over symmetric matrices, we have
\[
\log\det(\bX +\bDelta) - \log \deg(\bX) \le \tr(\bX^{-1} \bDelta) \le \fnorm{\bX^{-1}} \fnorm{\bDelta}.
\]
Let $\bDelta = \bW - \bar{\bW}$. We have 
$
    |\bDelta_{ij}| = | \ip{\ba_i, \ba_j} - \ip{\bar{\ba}_i, \bar{\ba}_j}| \le 3C \eps
$, 
hence $\fnorm{\bDelta} \le 3 C n \eps $
Then,
\begin{align*}
    \log \det(\bV) - \log\det(\bar{\bW}) &\stackrel{(a)}{=} \log \det(\bW) -  \log\det(\bar{\bW})\\
    &\le \tr(\bar{\bW}^{-1} \bDelta) \\
    &\le \sqrt{n} \opnorm{\bar{\bW}^{-1}} \fnorm{\bDelta} 
    \stackrel{(b)}{\le} \sqrt n \cdot 3 C n \eps.
\end{align*}
where (a) is by Sylvester's identity and (b) uses the fact that $\bar{\bW} \succeq \bI_n$, hence $\bar{\bW}^{-1} \preceq \bI_n$ giving $\opnorm{\bar{\bW}^{-1}} \le 1$. 
\end{proof}


\section{Technical Lemmas}
\label{appendix: technical lemmas}

In this Section, we provides the Proof for Lemmas in Appendix \ref{appendix: proof for lemmas in regret analysis} and other Technical Lemmas supporting the proofs. Most technical Lemmas are related to NTK and optimization in depp learning theory, mainly modified from the GNN helper Lemmas in \citep{kassraie2022graph} and technical Lemmas in \cite{zhou2020neural, vakili2021optimal}. 

\subsection{Notations for MLP}

Recall our GNN with one layer of linear graph convolution and a MLP:
\begin{equation}
\begin{aligned}
    f^{(1)}(\bh_i^G) &= \bW^{(1)} \bh_i^G, \quad i \in [N], \\
    f^{(l)}(\bh_i^G) &= \frac{1}{\sqrt m}
    \bW^{(l)} \relu(f^{(l-1)}(\bh_i^G)), \quad 2\leq l \leq L,\\
    \fmlp(\bh_i^G;\btheta) &= f^{(L)}(\bh_i^G) \\
    \fgnn(G; \btheta) &= \frac{1}{N} \sum_{i=1}^{N} \fmlp(\bh_i^G;\btheta).
\end{aligned}
\end{equation} 
We denote the gradients for GNN and associated MLP as 
\begin{equation}
\begin{aligned}
    \bcg(G;\btheta) &:= \nabla_{\btheta} \fgnn(G; \btheta) \\
    \gmlp(\cdot;\btheta) &:= \nabla_{\btheta} \fmlp(\cdot;\btheta)
\end{aligned}
\end{equation}
and the connection between gradients for the MLP and the gradient for the whole GNN is
\begin{equation}
    \bcg(G;\btheta) = \frac{1}{N} \sum_{i=1}^N \gmlp(\bh_i^G;\btheta)
\end{equation}

Similarly, we define a tangent kernel for the a MLP as 
\begin{equation}
    \tilde{k}_{MLP}(\bx, \bx^{\prime}) := \bcg_{MLP}(G;\btheta_0)^{\top} \bcg_{MLP}(G^{\prime};\btheta_0)
\end{equation}
for any MLP inputs $\bx$, $\bx^{\prime}$ and the associated neural tangent kernel $k_{MLP}(\bx, \bx^{\prime})$ is defined as limiting kernel of $\tilde{k}_{MLP}(\bx, \bx^{\prime})/m$:
\begin{equation}
    k_{MLP}(\bx, \bx^{\prime}) := \lim_{m \to \infty} \tilde{k}_{MLP}(\bx, \bx^{\prime})/m.
\end{equation}
By the connection between $\fgnn$ and $\fmlp$, we have
\begin{equation}
    k(G, G^{\prime}) = \frac{1}{N^2} \sum_{i \in \mV(G)} \sum_{j \in \mV(G^{\prime})} \kmlp(\bh^G_i, \bh^{G^{\prime}}_j).
\end{equation} 

\subsection{Proof for Lemmas in Appendix \ref{appendix: proof for lemmas in regret analysis}}

\begin{proof}[Proof of Lemma~\ref{lemma: taylor approximation of gnn}]
By Lemma~\ref{lemma: bound for gradient norm}, with probability at least $1-\delta \in (0,1)$
\begin{equation}
\begin{aligned}
    &|\fgnn(G; \btheta_t^{(J)}) - \fgnn(G; \btheta_0) - \langle \bcg(G; \btheta_0), \btheta_t^{(J)} - \btheta_0 \rangle |  \\
    &\qquad \qquad  \leq 
    \frac{1}{N} \sum_{j \in \mV(G)} 
    |\fmlp(\bh_j^G; \btheta_t^{(J)}) - \fmlp(\bh_j^G; \btheta_0) - \langle \gmlp(\bh_j^G; \btheta_0), \btheta_t^{(J)} - \btheta_0 \rangle | \\
    &\qquad \qquad  \leq
    C_1 \tau^{4/3} L^3 \sqrt{m \log(m)} \\
    &\qquad \qquad  \leq
    C_1 (\tilde{C} \sqrt{(R^2 + \sigma_{\eps}^2)/m \lambda})^{4/3}
    L^3 \sqrt{m \log(m)}
\end{aligned}
\end{equation}
where the last inequality is from the choice of $\tau = \tilde{C} \sqrt{(R^2 +  \sigma_{\eps}^2)/m \lambda}$ such that $\norm{\btheta_t^{(J)} - \btheta_0}_2 \leq \tau$. Since $\tau \propto 1/\sqrt{m}$, it can be verified that technical condition \eqref{eq: technical condition for tau} in Lemma~\ref{lemma: bound for gradient norm} is satisfied when $m$ is large. Therefore, set $C_2 = C_1 \tilde{C}^{4/3}$,
\begin{equation}
\begin{aligned}
    |\fgnn(G; \btheta_t^{(J)}) - \fgnn(G; \btheta_0) - \langle \bcg(G; \btheta_0), \btheta_t^{(J)} - \btheta_0 \rangle | 
    \leq
    C_2 L^3 (\frac{R^2 + \sigma_{\eps}^2}{m \lambda})^{2/3} \sqrt{m \log(m)}.
\end{aligned}
\end{equation}
\end{proof}

\begin{proof}[Proof of Lemma~\ref{lemma: bound for gradient error inner product}]
In this proof, set $\delta_1=\delta_2=\delta/2$ where $\delta \in (0,1)$ is an arbitrary small real value. We introduce $\{\tilde{\btheta}_t^{(j)}\}_{j=1}^J$ be the gradient descent update sequence of the following proximal optimization  \citep{kassraie2022graph}:
\begin{equation}
\begin{aligned}
    \min_{\btheta} \frac{1}{2t} \sum_{i=1}^t (\langle \bcg(G_i;\btheta_0), \btheta - \btheta_0 \rangle - y_i)^2 + \frac{m\lambda}{2} \norm{\btheta}_2^2
\end{aligned}
\end{equation}
and $\{\btheta_t^{(j)}\}_{j=1}^J$ be the gradient descent update sequence of parameters of our primary optimization \eqref{eq: minimize l2 loss}. In GNN training step in algorithms, we let $\btheta_t:=\btheta_{t}^{(J)}$. Recall that $\bar{\bU}_t = \lambda \bI + \bar{\bG}_t \bar{\bG}_t^{\top} /m$. By Lemma~\ref{lemma: approximation from gntk}, with probability at least $1-\delta_1 \in (0,1)$,  $\bar{\bU}_t \preccurlyeq (\lambda+ \frac{3}{2}\rho_{\max})\bI$.  Therefore,
\begin{equation}
\begin{aligned}
    |\langle \bcg(G;\btheta_0), \btheta_{t}-\btheta_0 - \bar{\bU}_{t}^{-1}\bar{\bG}_{t} \by_{t}/m \rangle| 
    & \leq
    \norm{\bcg(G;\btheta_0)}_{\bar{\bU}_t^{-1}}
    \norm{\btheta_{t}-\btheta_0 - \bar{\bU}_{t}^{-1}\bar{\bG}_{t} \by_{t}/m}_{\bar{\bU}_t}
    \\
    & \leq
    \sqrt{\lambda+3\rho_{\max}/2}
    \norm{\bcg(G;\btheta_0)}_{\bar{\bU}_t^{-1}}
    \norm{\btheta_{t}-\btheta_0 - \bar{\bU}_{t}^{-1}\bar{\bG}_{t} \by_{t}/m}_{2}
    \\
    & \leq 
    \sqrt{\lambda+3\rho_{\max}/2}
    \norm{\bcg(G;\btheta_0)}_{\bar{\bU}_t^{-1}}
    (\norm{\tilde{\btheta}_{t}^{(J)}-\btheta_0 - \bar{\bU}_{t}^{-1}\bar{\bG}_{t} \by_{t}/m}_{2} \\
    &\qquad \qquad  +
    \norm{\tilde{\btheta}_{t}^{(J)}-\btheta_{t}}_{2}
\end{aligned}
\end{equation}
By Lemma~\ref{lemma: parameter bound for proximal optimization } and Lemma~\ref{lemma: parameter bound for primary optimization }, with probability at least $1-\delta_2 \in (0,1)$, for some constants $C_1$ and $C_2$, we have
\begin{equation}
\begin{aligned}
    & |\langle \bcg(G;\btheta_0), \btheta_{t}-\btheta_0 - \bar{\bU}_{t}^{-1}\bar{\bG}_{t} \by_{t}/m \rangle| \\
    \leq &
    \sqrt{\lambda+3\rho_{\max}/2}
    \norm{\bcg(G;\btheta_0)}_{\bar{\bU}_t^{-1}}
    \bigg(
    C_1 (2-\eta m \lambda)^J  \sqrt{\frac{\sigma_{\eps}^2 + R^2}{m \lambda }} + 
    \norm{\tilde{\btheta}_{t}^{(J)}-\btheta_{t}}_{2}
    \bigg) 
    \quad \text{ (by Lemma~\ref{lemma: parameter bound for proximal optimization })} \\
    \leq & 
    \sqrt{\lambda+3\rho_{\max}/2}
    \norm{\bcg(G;\btheta_0)}_{\bar{\bU}_t^{-1}} \times \bigg(
    C_1 (2-\eta m \lambda)^J  \sqrt{\frac{\sigma_{\eps}^2 + R^2}{m \lambda }} + 
    C_2 \sqrt{\frac{\sigma_{\eps}^2 + R^2}{m \lambda }}
    \bigg)
    \quad \text{ (by Lemma~\ref{lemma: parameter bound for primary optimization })} \\
    = & 
    \sqrt{m(1+\frac{3 \rho_{\max} }{2 \lambda})}
    (C_1 (2-\eta m \lambda)^J + C_2 ) \sqrt{\frac{\sigma_{\eps}^2 + R^2}{m \lambda }} \bar{\sigma}_{t+1}(G)
\end{aligned}
\end{equation}
The last equality is obtained from the definition of $\bar{\sigma}_{t+1}^2(G)$, which is $\bar{\sigma}_{t+1}^2(G)= \lambda \bcg^{\top}(G;\btheta_0) \bar{\bU}_{t}^{-1} \bcg(G;\btheta_0)/m = \frac{\lambda}{m}\norm{\bcg(G;\btheta_0)}_{\bar{\bU}_t^{-1}}^2$. Now we let $\tilde{C} = \sqrt{m(1+\frac{3 \rho_{\max} }{2 \lambda})}(C_1 (2-\eta m \lambda)^J + C_2 ) \sqrt{\frac{\sigma_{\eps}^2 + R^2}{m \lambda }} $. Note that this constant $\tilde{C} = \mO(1)$ with respect to $m$ since $\eta = \mO(m^{-1})$. Then we have the desired result:
\begin{equation}
\begin{aligned}
    |\langle \bcg(G;\btheta_0), \btheta_{t}-\btheta_0 - \bar{\bU}_{t}^{-1}\bar{\bG}_{t} \by_{t}/m \rangle| \leq \tilde{C} \bar{\sigma}_{t+1}(G)
\end{aligned}
\end{equation}
where $\tilde{C} = (C_1 (2-\eta m \lambda)^J + C_2 ) \sqrt{\frac{\sigma_{\eps}^2 + R^2}{\lambda } (1+\frac{3 \rho_{\max} }{2 \lambda})} $ with $C_1 = \mO(1)$ and $C_2 = \mO(\lambda^{1/3})$.
\end{proof}

\begin{proof}[Proof of Lemma~\ref{lem:mu:is:linear}]
See Appendix~\ref{appendix: lemmas for gntk}.
\end{proof}

\begin{proof}[Proof of Lemma~\ref{lemma: high probability bound for sigma difference}]
Define function $\psi_{\lambda}$ for vectors $\{\bv, \ba_1, ..., \ba_{t-1}\}$ as followed:
\begin{equation}
    \psi_{\lambda}(\bv, \ba_1, ..., \ba_{t-1}) := \sqrt{\bv^{\top} (\lambda \bI + \sum_{i=1}^{t-1}\ba_i \ba_i^{\top})^{-1}\bv} ,
\end{equation}
and denote the gradients for $\psi_{\lambda}$ as 
\begin{equation}
\begin{aligned}
    \nabla_{0} \psi_{\lambda} &:= \nabla_{\bv} \psi_{\lambda}(\bv, \ba_1, ..., \ba_{t-1}) \\
    \nabla_{i} \psi_{\lambda} &:= \nabla_{\ba_i} \psi_{\lambda}(\bv, \ba_1, ..., \ba_{t-1}), \forall i \in [t-1]. 
\end{aligned}
\end{equation}
By setting $\bA = (\lambda \bI + \sum_{i=1}^{t-1} \ba_i \ba_i^{\top})^{-1} \preccurlyeq \frac{1}{\lambda} \bI$ with eigendecomposition $\bA = \bV \bD \bV^{\top}$. The gradients are bounded as followed
\begin{equation}
\label{eq in lemma proof: gradient of sigma norm bound}
\begin{aligned}
    \norm{\nabla_{0} \psi_{\lambda}}_2 
    &= \frac{\norm{\bA \bv}_2}{\sqrt{\bv^{\top} \bA \bv}}
    = \sqrt{\frac{\bv^{\top} \bA^2 \bv}{\bv^{\top} \bA \bv}}
    \leq \sqrt{\lambda_{max}(\bA)} 
    \leq 1/\sqrt{\lambda}
    \\
    \norm{\nabla_{i} \psi_{\lambda}}_2
    &= \frac{\norm{\bA \bv\bv^{\top} \bA \ba_i}_2}{\sqrt{\bv^{\top} \bA \bv}} 
    \leq \norm{\ba_i}_2 \frac{\bv^{\top} \bA^2 \bv}{\sqrt{\bv^{\top} 
    \bA \bv}} 
    \leq 
    \norm{\ba_i}_2 \norm{\bv}_2 / \lambda
\end{aligned}
\end{equation}

We can express $\bar{\sigma}_{t}(G)$ and $\sigma_{t}(G)$ by $\psi_{\lambda}$:
\begin{equation}
\begin{aligned}
    \bar{\sigma}_{t}(G) &= \psi_{\lambda}(\frac{\bcg(G;\btheta_{t-1})}{\sqrt{m}}, \frac{\bcg(G_1;\btheta_{1})}{\sqrt{m}},..., \frac{\bcg(G_{t-1};\btheta_{t-1})}{\sqrt{m}}) \\
    \sigma_{t}(G) &= \psi_{\lambda}(\frac{\bcg(G;\btheta_{0})}{\sqrt{m}}, \frac{\bcg(G_1;\btheta_{0})}{\sqrt{m}},..., \frac{\bcg(G_{t-1};\btheta_{0})}{\sqrt{m}}).
\end{aligned}
\end{equation}
From Lemma~\ref{lemma: bound for gradient norm}, there exists positive constants such that the gradients and gradient differences are bounded with high probability, which indicates for some constant $C_1$ with probability greater than $1-\delta$, 
\begin{equation}
\begin{aligned}
\label{eq in lemma proof: bound for gradient}
    \norm{\bcg(G;\btheta)}_2 = \norm{\frac{1}{N} \sum_{j \in \mV(G)} \gmlp(\bh_j^G; \btheta)}_2 \leq C_1 \sqrt{m L}
\end{aligned}
\end{equation}
Note that $\psi_{\lambda}$ is Lipschitz continuous, then with high probability, we have
\begin{equation}
\begin{aligned}
    |\bar{\sigma}_{t}(G)-\sigma_{t}(G)| 
    &= 
    |\psi_{\lambda}(\frac{\bcg(G;\btheta_{t-1})}{\sqrt{m}}, \frac{\bcg(G_1;\btheta_{1})}{\sqrt{m}},..., \frac{\bcg(G_{t-1};\btheta_{t-1})}{\sqrt{m}}) - \psi_{\lambda}(\frac{\bcg(G;\btheta_{0})}{\sqrt{m}}, \frac{\bcg(G_1;\btheta_{0})}{\sqrt{m}},..., \frac{\bcg(G_{t-1};\btheta_{0})}{\sqrt{m}})| \\
    & \leq 
    \sup \{\norm{\nabla_{0} \psi_{\lambda}}_2 \}
    \norm{\frac{\bcg(G;\btheta_{t-1})}{\sqrt{m}} - \frac{\bcg(G;\btheta_{0})}{\sqrt{m}}}_2 +
    \sum_{i=1}^{t-1} \sup \{\norm{\nabla_{i} \psi_{\lambda}}_2 \}
    \norm{\frac{\bcg(G_{i};\btheta_{i})}{\sqrt{m}} - \frac{\bcg(G_i;\btheta_{0})}{\sqrt{m}}}_2 \\
    & \leq 
    \frac{1}{\sqrt{\lambda}}
    \norm{\frac{\bcg(G;\btheta_{t-1})}{\sqrt{m}} - \frac{\bcg(G;\btheta_{0})}{\sqrt{m}}}_2 +
    \frac{C_1^2 L}{\lambda} 
    \sum_{i=1}^{t-1}
    \norm{\frac{\bcg(G_{i};\btheta_{i})}{\sqrt{m}} - \frac{\bcg(G_i;\btheta_{0})}{\sqrt{m}}}_2 
    \text{( by \eqref{eq in lemma proof: gradient of sigma norm bound} and \eqref{eq in lemma proof: bound for gradient})} \\
    & \leq 
    C_2 \sqrt{\log(m)} \tau^{1/3} L^3 \norm{\bcg(G;\btheta_0)}_2 /\sqrt{m}
    (\frac{1}{\sqrt{\lambda}} + \frac{C_1^2 L t}{\lambda}) 
    \quad \text{(by Lemma~\ref{lemma: bound for gradient norm})} \\
    & \leq
    C_1 C_2 \sqrt{\log(m)} \tau^{1/3} L^{7/2} 
    (\frac{1}{\sqrt{\lambda}} + \frac{C_1^2 L t}{\lambda}) 
    \quad \text{(by \eqref{eq in lemma proof: bound for gradient})}
\end{aligned}
\end{equation}
Therefore, if $\lambda \leq C_1^4 L^2 t^2$ and let $\tau = \tilde{C} \sqrt{\frac{R^2 + \sigma_{\eps}^2}{m \lambda}}$, $ C_3 = 2 \tilde{C} C_2 C_1^3 $,
\begin{equation}
\begin{aligned}
    |\bar{\sigma}_{t}(G)-\sigma_{t}(G)| 
    \leq 
    C_3 t
    \lambda^{-7/6} L^{9/2} (R^2 + \sigma_{\eps}^2)^{1/6} m^{-1/6} \sqrt{\log(m)}
\end{aligned}
\end{equation}
\end{proof}

\subsection{Lemmas for GNN training }

\begin{lemma} [Parameter Bound for Primary Optimization]
\label{lemma: parameter bound for primary optimization }
Let $\{\btheta_t^{(j)}\}_{j=1}^J$ be the gradient descent update sequence of parameters of the optimization \eqref{eq: minimize l2 loss} which is,
\begin{equation}
\begin{aligned}
    \min_{\btheta} \frac{1}{2t}\sum_{i=1}^t (\fgnn(G_i; \btheta) - y_i)^2 +  \frac{m\lambda}{2} \norm{\btheta}_2^2
\end{aligned}
\end{equation}
then if $m \ge poly(R, \sigma_{\eps}, L, \lambda^{-1} , \log(\frac{N}{\delta}))$ and learning rate $\eta \leq (\tilde{C} mL+m\lambda)^{-1}$ for some constant $\tilde{C}$. Then for a constant $C = \mO(\lambda^{1/3})$ which is independent of $m$ and $t$, with probability at least $1-\delta$
\begin{equation}
\begin{aligned}
    \norm{\btheta_{t}^{(j)} - \tilde{\btheta}_t^{(j)}}_2 \leq 
    C \sqrt{\frac{R^2 + \sigma_{\eps}^2}{m \lambda}}
\end{aligned}
\end{equation}
where $\{\tilde{\btheta}_t^{(j)}\}_{j=1}^J$ be the gradient descent update sequence of parameters of the proximal optimization with loss function $\frac{1}{2t} \sum_{i=1}^t (\langle \bcg(G_i;\btheta_0), \btheta - \btheta_0 \rangle - y_i)^2 + \frac{m\lambda}{2} \norm{\btheta}_2^2$. Both optimization have the same initialization at $\tilde{\btheta}_t^{(0)} = \btheta_t^{(0)} = \btheta_0$ and same learning rate $\eta$.
\end{lemma}
\begin{proof} 
In this proof, set $\delta_1=\delta_2=\delta/2$ where $\delta \in (0,1)$ is an arbitrary small real value. Define $\bG_t^{(j)} := [\bcg(G_1;\btheta_t^{(j)}),...,\bcg(G_t;\btheta_t^{(j)}))] \in \BR^{p \times t}$ as the $j$-th updates in our primary optimzation with loss \eqref{eq: minimize l2 loss} at round $t$. Also define $\bbf_{gnn, t}^{(j)} := [\fgnn(G_1;\btheta_t^{(j)}), ..., \fgnn(G_t;\btheta_t^{(j)})]^{\top} \in \BR^{t \times 1}$. The gradient descent updates for sequences $\{\btheta_t^{(j)}\}_{j=1}^J$ and $\{\tilde{\btheta}_t^{(j)}\}_{j=1}^J$ are
\begin{equation}
\begin{aligned}
    \btheta_{t}^{(j+1)} &= \btheta_{t}^{(j)} - 
    \eta \bigg( \frac{1}{t} [\bG_t^{(j)}]^{\top}(\bbf_{gnn, t}^{(j)}-\by_t) + m \lambda \btheta_{t}^{(j)} \bigg)\\
    \tilde{\btheta}_t^{(j+1)} &= \tilde{\btheta}_t^{(j)} - \eta \bigg( \frac{1}{t} \bar{\bG}_t^{\top} (\bar{\bG}_t(\tilde{\btheta}_t^{(j)}-\btheta_0) - \by_t) + m \lambda \tilde{\btheta}_t^{(j)} \bigg)
\end{aligned}
\end{equation}
Therefore,
\begin{equation}
\begin{aligned}
    & \norm{\btheta_{t}^{(j+1)} - \tilde{\btheta}_t^{(j+1)}}_2 \\
    =
    & \norm{ (1-\eta m\lambda)(\btheta_{t}^{(j)} - \tilde{\btheta}_t^{(j)}) - \frac{\eta}{t} [\bG_t^{(j)}]^{\top}(\bbf_{gnn, t}^{(j)}-\by_t) + \frac{\eta}{t} \bar{\bG}_t^{\top} (\bar{\bG}_t(\tilde{\btheta}_t^{(j)}-\btheta_0) - \by_t) }_2 \\
    = 
    & \norm{(1-\eta m\lambda)(\btheta_{t}^{(j)} - \tilde{\btheta}_t^{(j)}) - \frac{\eta}{t} (\bG_t^{(j)} - \bar{\bG}_t)^{\top}(\bbf_{gnn, t}^{(j)}-\by_t) -\frac{\eta}{t}  \bar{\bG}_t^{\top} (\bbf_{gnn, t}^{(j)} - \bar{\bG}_t(\tilde{\btheta}_t^{(j)}-\btheta_0))}_2 \\
    = 
    & \norm{(1-\eta m\lambda)(\btheta_{t}^{(j)} - \tilde{\btheta}_t^{(j)}) - \frac{\eta}{t} (\bG_t^{(j)} - \bar{\bG}_t)^{\top}(\bbf_{gnn, t}^{(j)}-\by_t) -\frac{\eta}{t}  \bar{\bG}_t^{\top} (\bbf_{gnn, t}^{(j)} - \bar{\bG}_t(\btheta_{t}^{(j)}-\btheta_0) + \bar{\bG}_t(\btheta_{t}^{(j)}-\tilde{\btheta}_t^{(j)}))}_2 \\
    = 
    & \norm{( \bI - \eta(m \lambda \bI + \bar{\bG}_t^{\top}\bar{\bG}_t/t))(\btheta_{t}^{(j)} - \tilde{\btheta}_t^{(j)}) - \frac{\eta}{t} (\bG_t^{(j)} - \bar{\bG}_t)^{\top}(\bbf_{gnn, t}^{(j)}-\by_t) - \frac{\eta}{t}  \bar{\bG}_t^{\top} (\bbf_{gnn, t}^{(j)} - \bar{\bG}_t(\btheta_{t}^{(j)}-\btheta_0))}_2 \\
    \leq &
    \underbrace{\norm{( \bI - \eta(m \lambda \bI + \bar{\bG}_t^{\top}\bar{\bG}_t/t))}_2
    \norm{\btheta_{t}^{(j)} - \tilde{\btheta}_t^{(j)}}_2}_{I_1}
    +
    \underbrace{\frac{\eta}{t}
    \norm{\bar{\bG}_t}_2
    \norm{\bbf_{gnn, t}^{(j)} - \bar{\bG}_t(\btheta_{t}^{(j)}-\btheta_0)}_2}_{I_2}
    + 
    \underbrace{\frac{\eta}{t}
    \norm{\bG_t^{(j)} - \bar{\bG}_t}_2 
    \norm{\bbf_{gnn, t}^{(j)}-\by_t}_2}_{I_3}
\end{aligned}
\end{equation}
For $I_1$, due to $\bar{\bG}_t^{\top} \bar{\bG}_t /t \succcurlyeq \bZeros$, we have 
\begin{equation}
\begin{aligned}
    I_1 
    = \norm{( \bI - \eta(m \lambda \bI + \bar{\bG}_t^{\top}\bar{\bG}_t/t))}_2
    \norm{\btheta_{t}^{(j)} - \tilde{\btheta}_t^{(j)}}_2 
    \leq (1-\eta m \lambda)
    \norm{\btheta_{t}^{(j)} - \tilde{\btheta}_t^{(j)}}_2
\end{aligned}
\end{equation}
For $I_2$, by Lemma~\ref{lemma: gradient descent norm bound},
set $\tau = \tilde{C} \sqrt{(R^2 + \sigma_{\eps}^2)/m \lambda}$. Since $\tau \propto 1/\sqrt{m}$, it can be verified that technical condition \eqref{eq: technical condition for tau} in Lemma~\ref{lemma: bound for gradient norm} is satisfied when $m$ is large. Then with probability at least $1-\delta_1 \in (0,1)$, 
\begin{equation}
\begin{aligned}
    I_2 = 
    \frac{\eta}{t}
    \norm{\bar{\bG}_t}_2
    \norm{\bbf_{gnn, t}^{(j)} - \bar{\bG}_t(\btheta_{t}^{(j)}-\btheta_0)}_2 
    \leq
    \eta C_1 (\tilde{C} \frac{R^2 + \sigma_{\eps}^2}{m \lambda})^{2/3} L^{7/2} m \sqrt{\log(m)}
\end{aligned}
\end{equation}
For $I_3$, by Lemma~\ref{lemma: prediction error bound in gradient descent} and Lemma~\ref{lemma: gradient descent norm bound}, and Lemma~\ref{lemma: bound for gradient norm}, with probability at least $1-\delta_2 \in (0,1)$, 
\begin{equation}
\begin{aligned}
    I_3 &=
    \frac{\eta}{t}
    \norm{\bG_t^{(j)} - \bar{\bG}_t}_2 
    \norm{\bbf_{gnn, t}^{(j)}-\by_t}_2
    \leq 
    \eta
    C_2 (\tilde{C} \frac{R^2 + \sigma_{\eps}^2}{m \lambda})^{1/6} L^{7/2} \sqrt{ m \log(m)}
    \sqrt{R^2 + \sigma_{\eps}^2}
\end{aligned}
\end{equation}
Put the upper bound for $I_1$, $I_2$., $I_3$ together and set $C_3 = (\lambda^{1/3} C_1 + C_2 )\tilde{C} = \mO(\lambda^{1/3})$, then we get,
\begin{equation}
\begin{aligned}
    \norm{\btheta_{t}^{(j+1)} - \tilde{\btheta}_t^{(j+1)}}_2 
    \leq 
    (1-\eta m \lambda) \norm{\btheta_{t}^{(j)} - \tilde{\btheta}_t^{(j)}}_2
    + C_3 \eta (R^2 + \sigma_{\eps}^2)^{2/3} L^{7/2} m^{1/3} \lambda^{-1/6} \sqrt{\log(m)} 
\end{aligned}
\end{equation}
Therefore, there exists $m=poly(R, \sigma_{\eps}, \lambda , L)$ satisfies that $(R^2 + \sigma_{\eps}^2)^{1/6} L^{7/2} \lambda^{1/3} \sqrt{\log(m)} \leq m^{1/6}$, which indicates
\begin{equation}
\begin{aligned}
    \norm{\btheta_{t}^{(j)} - \tilde{\btheta}_t^{(j)}}_2 
    & \leq 
    C_3  (R^2 + \sigma_{\eps}^2)^{2/3} L^{7/2} m^{-2/3} \lambda^{-1/6} \sqrt{\log(m)} 
    & \leq 
    C_3 \sqrt{\frac{R^2 + \sigma_{\eps}^2}{m \lambda}}
\end{aligned}
\end{equation}
\end{proof}

\begin{lemma} [Prediction Error Bound in Gradient Descent]
\label{lemma: prediction error bound in gradient descent}
Let $\{\btheta_t^{(j)}\}_{j=1}^J$ be the gradient descent update sequence of parameters of the optimization \eqref{eq: minimize l2 loss}. Define $\bbf_{gnn, t}^{(j)} := [\fgnn(G_1;\btheta_t^{(j)}), ..., \fgnn(G_t;\btheta_t^{(j)})]^{\top} \in \BR^{t \times 1}$. Assume $\tau$ is set such that $\norm{\btheta_t^{(j)} - \btheta_0}_2 \leq \tau$ for all $t$ and $\forall j \leq J$. Suppose $m \ge poly(L, \lambda^{-1}, \log(N/\delta))$ where $\delta \in (0,1)$ and learning rate $\eta \leq (\tilde{C} mL+m\lambda)^{-1}$ for some constant $\tilde{C}$, then with probability at least $1-\delta$,
\begin{equation}
\begin{aligned}
    \norm{\bbf_{gnn, t}^{(j)}-\by_t}_2 \leq C \sqrt{t(R^2 + \sigma_{\eps}^2)}
\end{aligned}
\end{equation}
where $C$ is some constant which does not depend on $m$ and $t$.
\end{lemma}
\begin{proof}
Define $\bbf_t(\btheta)$ and $\bG_t(\btheta)$ as follow
\begin{equation}
\begin{aligned}
    \bbf_t(\btheta) &= [\fgnn(G_1;\btheta), ..., \fgnn(G_t;\btheta)]^{\top} \in \BR^{t \times 1}\\
    \bG_t(\btheta) &= [\bcg(G_1;\btheta),...,\bcg(G_t;\btheta)] \in \BR^{p \times t}
\end{aligned}
\end{equation}
Also define $\mL_t(\btheta):=\frac{1}{2t}\sum_{i=1}^t (\fgnn(G_i; \btheta) - y_i)^2 +  \frac{m\lambda}{2} \norm{\btheta}_2^2$ as the loss function in primary optimization. Note that $\mL_t(\btheta):= \frac{1}{2t}\norm{\bbf_t(\btheta) - \by_t}_2^2 + \frac{m\lambda}{2}\norm{\btheta}_2^2$
First notice that loss function $\mL_t(\btheta)$ is convex due to the strongly convexity of $\norm{\cdot}_2^2/2$. We are going to use the following two-sided bound from strongly convexity in this proof: 
\begin{equation}
\begin{aligned}
    \norm{\by}_2^2/2 - \norm{\bx}_2^2/2 =
    \bx^{\top}(\by - \bx) + \frac{1}{2} \norm{\by - \bx}_2^2 
\end{aligned}
\end{equation}
By $1$-strongly convexity of $\norm{\cdot}_2^2/2$, we have
\vspace{-5mm}
\begin{equation}
\begin{aligned}
    \mL_t(\btheta^{\prime}) - \mL_t(\btheta)
    = & 
    \frac{1}{2t}\bigg( \norm{\bbf_t(\btheta^{\prime}) - \by_t}_2^2 - \norm{\bbf_t(\btheta) - \by_t}_2^2 \bigg) + \frac{m \lambda}{2} \bigg( \norm{\btheta^{\prime}}_2^2 - \norm{\btheta}_2^2 \bigg) \\
    \leq &
    \frac{1}{t}\bigg( ( \bbf_t(\btheta) - \by_t)^{\top}(\bbf_t(\btheta^{\prime}) - \bbf_t(\btheta)) + \frac{1}{2} \norm{\bbf_t(\btheta) - \bbf_t(\btheta^{\prime})}_2^2 \bigg) + m \lambda \bigg( \btheta^{\top} (\btheta^{\prime} - \btheta) + \frac{1}{2} \norm{\btheta-\btheta^{\prime}}_2^2 \bigg).
\end{aligned}
\end{equation}
Define $\be_t:= \bbf_t(\btheta^{\prime}) - \bbf_t(\btheta) - \bG_t^{\top}(\btheta)(\btheta^{\prime}- \btheta)$. By Lemma~\ref{lemma: gradient descent norm bound}, with probability at least $1-\delta_1 \in (0,1)$ 
\vspace{-5mm}
\begin{equation}
\begin{aligned}
\label{upper bound in prediction error bound in gradient descent}
    & \mL_t(\btheta^{\prime}) - \mL_t(\btheta) \\
    \leq &
    \frac{1}{t} 
    ( \bbf_t(\btheta) - \by_t)^{\top}(\bG_t^{\top}(\btheta)(\btheta^{\prime}- \btheta) + \be_t) 
    + \frac{1}{2t} \norm{\bG_t^{\top}(\btheta)(\btheta^{\prime}- \btheta) + \be_t}_2^2 
    + m \lambda \bigg( \btheta^{\top} (\btheta^{} - \btheta) + \frac{1}{2} \norm{\btheta-\btheta^{\prime}}_2^2 \bigg) \\
    =&
    \frac{1}{t} 
    [\bG_t(\btheta)( \bbf_t(\btheta) - \by_t) + m \lambda  \btheta]^{\top}(\btheta^{\prime}- \btheta)  + 
    \frac{1}{t}
    ( \bbf_t(\btheta) - \by_t)^{\top} \be_t +
    \frac{1}{2t} \norm{\bG_t^{\top}(\btheta)(\btheta^{\prime}- \btheta) + \be_t}_2^2 + \frac{m \lambda}{2} \norm{\btheta-\btheta^{\prime}}_2^2  \\
    =& 
    \nabla \mL_t(\btheta)^{\top}(\btheta^{\prime}- \btheta)  + 
    \frac{1}{t} 
    ( \bbf_t(\btheta) - \by_t)^{\top} \be_t 
    + \frac{1}{2t} \norm{\bG_t^{\top}(\btheta)(\btheta^{\prime}- \btheta) + \be_t}_2^2
    + \frac{m \lambda}{2} \norm{\btheta-\btheta^{\prime}}_2^2 \\
    \leq &
    \nabla \mL_t(\btheta)^{\top}(\btheta^{\prime}- \btheta)  + 
    \frac{1}{t} 
    \norm{\bbf_t(\btheta) - \by_t}_2 \norm{\be_t }_2
    + 
    \frac{1}{t} \norm{\bG_t(\btheta)}_2^2 \norm{\btheta^{\prime}- \btheta}_2^2
    +
    \frac{1}{t} \norm{\be_t}_2^2
    + \frac{m \lambda}{2} \norm{\btheta-\btheta^{\prime}}_2^2 \\
    \leq &
    \nabla \mL_t(\btheta)^{\top}(\btheta^{\prime}- \btheta)  + 
    \frac{1}{t} 
    \norm{\bbf_t(\btheta) - \by_t}_2 \norm{\be_t }_2
    +
    \frac{1}{t} \norm{\be_t}_2^2 
    + 
    ( C_1^2 m L + m \lambda /2) \norm{\btheta^{\prime}- \btheta}_2^2 \quad \text{ (by Lemma~\ref{lemma: gradient descent norm bound})}
\end{aligned}
\end{equation}
Similarly by $1$-strongly convexity of $\norm{\cdot}_2^2/2$ , we also investigate the lower bound:
\vspace{-3mm}
\begin{equation}
    \mL_t(\btheta^{\prime}) - \mL_t(\btheta)
    \geq
    \frac{1}{t}\bigg( ( \bbf_t(\btheta) - \by_t)^{\top}(\bbf_t(\btheta^{\prime}) - \bbf_t(\btheta)) 
    + \frac{1}{2} \norm{\bbf_t(\btheta^{\prime}) - \bbf_t(\btheta_)}_2^2 \bigg) +  m \lambda \bigg( \btheta^{\top} (\btheta^{\prime} - \btheta) + \frac{1}{2} \norm{\btheta^{\prime}-\btheta}_2^2 \bigg)
\end{equation}
Using $\be_t:= \bbf_t(\btheta^{\prime}) - \bbf_t(\btheta) - \bG_t^{\top}(\btheta)(\btheta^{\prime}- \btheta)$, we obtain
\vspace{-3mm}
\begin{equation}
\begin{aligned}
    \mL_t(\btheta^{\prime}) - \mL_t(\btheta)
    \geq &
    \frac{1}{t} 
    ( \bbf_t(\btheta) - \by_t)^{\top}(\bG_t^{\top}(\btheta)(\btheta^{\prime}- \btheta) + \be_t) 
    +
    m \lambda  \btheta^{\top} (\btheta^{\prime} - \btheta) + \frac{m \lambda}{2} \norm{\btheta-\btheta^{\prime}}_2^2  \\
    =&
    \frac{1}{t} 
    [\bG_t(\btheta)( \bbf_t(\btheta) - \by_t) + m \lambda  \btheta]^{\top}(\btheta^{\prime}- \btheta)  
    + \frac{1}{t} 
    ( \bbf_t(\btheta) - \by_t)^{\top} \be_t 
    + \frac{m \lambda}{2} \norm{\btheta-\btheta^{\prime}}_2^2
\end{aligned}
\end{equation}
Then using $\nabla \mL_t(\btheta) = \bG_t(\btheta)( \bbf_t(\btheta) - \by_t) + m \lambda  \btheta$, we have
\vspace{-3mm}
\begin{equation}
\begin{aligned}
\label{lower bound in prediction error bound in gradient descent}
    \mL_t(\btheta^{\prime}) - \mL_t(\btheta) 
    \geq &
    \nabla \mL_t(\btheta)^{\top}(\btheta^{\prime}- \btheta)  + \frac{1}{t} 
    ( \bbf_t(\btheta) - \by_t)^{\top} \be_t 
    + \frac{m \lambda}{2} \norm{\btheta-\btheta^{\prime}}_2^2 \\
    \geq &
    \nabla \mL_t(\btheta)^{\top}(\btheta^{\prime}- \btheta)  
    + \frac{m \lambda}{2} \norm{\btheta-\btheta^{\prime}}_2^2
    - \frac{1}{t} 
    \norm{\bbf_t(\btheta) - \by_t}_2 
    \norm{\be_t}_2 \\
    \geq &
    -\frac{\norm{\nabla \mL_t(\btheta)}_2^2}{2 m \lambda}
    - \frac{1}{t} 
    \norm{\bbf_t(\btheta) - \by_t}_2 
    \norm{\be_t}_2 
    \quad \text{ (by Lemma~\ref{lemma: algebra lemma})}
\end{aligned}
\end{equation}
Now recall the update step $\btheta_t^{(j+1)} = \btheta_t^{(j)} - \eta \nabla \mL_t(\btheta_t^{(j)})$ and combine the above upper and lower bounds,
\begin{equation}
\begin{aligned}
\label{combined upper bound in prediction error bound in gradient descent}
    & \mL_t(\btheta - \eta \nabla \mL_t(\btheta)) - \mL_t(\btheta) \\
    \leq &
    - \eta 
    \norm{\nabla \mL_t(\btheta)}_2^2  + 
    \frac{1}{t} 
    \norm{\bbf_t(\btheta) - \by_t}_2 \norm{\be_t }_2
    +
    \frac{1}{t} \norm{\be_t}_2^2 
    + 
    \eta^2 ( C_1^2 m L + m \lambda /2) \norm{\nabla \mL_t(\btheta)}_2^2 
    \quad \text{ (by update step and \eqref{upper bound in prediction error bound in gradient descent})} \\
    = &
    -\eta 
    \bigg(1 - \frac{\eta}{2} (2 C_1^2 m L + m \lambda )
    \bigg) \norm{\nabla \mL_t(\btheta)}_2^2
    + 
    \frac{1}{t} 
    \norm{\bbf_t(\btheta) - \by_t}_2 \norm{\be_t }_2
    +
    \frac{1}{t} \norm{\be_t}_2^2 
    \\
    \leq &
    - \frac{\eta}{2} \norm{\nabla \mL_t(\btheta)}_2^2
    + 
    \frac{1}{t} 
    \norm{\bbf_t(\btheta) - \by_t}_2 \norm{\be_t }_2
    +
    \frac{1}{t} \norm{\be_t}_2^2 
    \quad \text{ (by choice of $\eta$)}\\
    \leq & 
    \eta m \lambda
    \bigg( \mL_t(\btheta^{\prime}) - \mL_t(\btheta) +  \frac{1}{t}\norm{\bbf_t(\btheta) - \by_t}_2 
    \norm{\be_t}_2 \bigg) 
    +   
    \frac{1}{t} 
    \norm{\bbf_t(\btheta) - \by_t}_2 \norm{\be_t }_2
    +
    \frac{1}{t} \norm{\be_t}_2^2 
    \quad \text{ (by \eqref{lower bound in prediction error bound in gradient descent})} \\
    \leq & 
    \eta m \lambda
    \bigg( \mL_t(\btheta^{\prime}) - \mL_t(\btheta) +  \norm{\bbf_t(\btheta) - \by_t}_2^2/8t +
    2\norm{\be_t}_2^2/t \bigg) 
    +   
    \frac{1}{t} 
    (\eta m \lambda \norm{\bbf_t(\btheta) - \by_t}_2^2/8 + 2 \norm{\be_t }_2^2/\eta m \lambda )
    +
    \frac{1}{t} \norm{\be_t}_2^2 \\
    = & 
    \eta m \lambda (\mL_t(\btheta^{\prime}) - \mL_t(\btheta)) 
    + \frac{\eta m \lambda}{4t} \norm{\bbf_t(\btheta) - \by_t}_2^2 
    + (\frac{2 \eta m \lambda }{t} + \frac{2}{ \eta m \lambda t} + \frac{1}{t}) \norm{\be_t}_2^2 \\
    \leq & 
    \eta  m \lambda (\mL_t(\btheta^{\prime}) - \mL_t(\btheta)) 
    + \eta m \lambda \mL_t(\btheta)/2 
    + (\frac{2 \eta m \lambda }{t} + \frac{2}{ \eta m \lambda t} + \frac{1}{t}) \norm{\be_t}_2^2
    \quad \text{ (by $\norm{\bbf_t(\btheta) - \by_t}_2^2 \leq 2t \mL(\btheta)$)} \\
    = & 
    \eta  m \lambda ( \mL_t(\btheta^{\prime}) - \mL_t(\btheta)/2)
    + (\frac{2 \eta m \lambda }{t} + \frac{2}{ \eta m \lambda t} + \frac{1}{t}) \norm{\be_t}_2^2
\end{aligned}
\end{equation}
For $\norm{\be_t}_2^2$, by Lemma~\ref{lemma: bound for gradient norm}, with probability at least $1-\delta_2 \in (0,1)$ for some constant $C_2$, we have
\begin{equation}
\begin{aligned}
\label{e_t upper bound in prediction error bound in gradient descent}
    \norm{\be_t}_2
    &= \norm{\bbf_t(\btheta^{\prime}) - \bbf_t(\btheta) - \bG_t^{\top}(\btheta)(\btheta^{\prime}- \btheta)}_2 \\
    & \leq 
    \sqrt{t} \max_{i \in [t]}
    |\fgnn(G_i;\btheta^{\prime}) - \fgnn(G_i;\btheta) + \bcg^{\top}(G_i;\btheta)(\btheta^{\prime}- \btheta) | \\
    & \leq 
    \frac{\sqrt{t}}{N} \max_{i \in [t]} \sum_{j \in \mV(G_i)}
    |\fmlp(\bh_j;\btheta^{\prime}) - \fmlp(\bh_j;\btheta) + \gmlp(\bh_j;\btheta)^{\top}(\btheta^{\prime}- \btheta) | \\
    & \leq
    C_2 \tau^{4/3} L^3 \sqrt{t m \log(m)}
\end{aligned}
\end{equation}
where $\mV(G)$ as vertice set of a graph $G$. Moreove, by Lemma~\ref{lemma: bound for subgaussian noise norm}, we have the high probability upper bound for $\frac{1}{t} \norm{\by_t}_2^2$: with probability at least $1-\delta_3 \in (0,1)$ and some constant $C_3$ depends on $\delta_3$, 
\begin{equation}
\begin{aligned}
\label{y_t upper bound in prediction error bound in gradient descent}
    \frac{1}{t} \norm{\by_t}_2^2 
    \leq 
    \frac{1}{t}
    (t R^2 + \norm{\bepsilon_t}_2^2 + 2 \sqrt{t} R \norm{\bepsilon_t}_2)
    \leq 
    C_3 (\sigma_{\eps}^2 + R^2)
\end{aligned}
\end{equation}
Then let $\btheta^{\prime} = \btheta_0$ and plug in $\btheta_t^{(j+1)}$ and $\btheta_t^{(j)}$ in \eqref{combined upper bound in prediction error bound in gradient descent}, by Lemma~\ref{lemma: parameter bound for proximal optimization }, with probability at least $1-\delta_4$,
\begin{equation}
\begin{aligned}
    \mL_t(\btheta_t^{(j+1)}) - \mL_t(\btheta_0)
    & \leq (1-\eta m \lambda/2) (\mL_t(\btheta_t^{(j)}) - \mL_t(\btheta_0)) + \frac{\eta m \lambda }{2} \mL_t(\btheta_0) + (\frac{2 \eta m \lambda }{t} + \frac{2}{ \eta m \lambda t} + \frac{1}{t}) \norm{\be_t}_2^2 \\
    & \leq 
    (1-\eta m \lambda/2) (\mL_t(\btheta_t^{(j)}) - \mL_t(\btheta_0)) + 
    \frac{\eta m \lambda }{2} (\frac{1}{t} \norm{\by_t}_2^2 + m\lambda \norm{\btheta_0}_2^2) \\
    & \quad +
    (2 \eta m \lambda  + 2/\eta m \lambda + 1) C_2^2 \tau^{8/3} L^6 m \log(m) 
    \quad \text{ (by \eqref{y_t upper bound in prediction error bound in gradient descent})} \\
    & \leq 
    (1-\eta m \lambda/2) (\mL_t(\btheta_t^{(j)}) - \mL_t(\btheta_0)) + 
    \frac{\eta m \lambda }{2} (C_3 (\sigma_{\eps}^2 + R^2)  + m\lambda \norm{\btheta_0}_2^2) \\
    & \quad +
    \frac{5}{\eta m \lambda} C_2^2 \tau^{8/3} L^6 m \log(m)
    \quad \text{ (by \eqref{e_t upper bound in prediction error bound in gradient descent} and $\eta m \lambda \leq 1$)} \\
    & \leq 
    (1-\eta m \lambda/2) (\mL_t(\btheta_t^{(j)}) - \mL_t(\btheta_0)) + 
    C_4 \eta m \lambda (\sigma_{\eps}^2 + R^2) +
    \frac{5}{\eta m \lambda} C_2^2 \tau^{8/3} L^6 m \log(m) \\
    & \quad \text{ (by Lemma~\ref{lemma: parameter bound for proximal optimization })}
\end{aligned}
\end{equation}
Now we further set $\tau = \tilde{C} \sqrt{\frac{\sigma_{\eps}^2 + R^2}{m \lambda}}$ and the upper bound for $\mL_t(\btheta_t^{(j+1)}) - \mL_t(\btheta_0)$ is
\begin{equation}
\begin{aligned}
    \mL_t(\btheta_t^{(j+1)}) - \mL_t(\btheta_0)
    & \leq 
    (1-\eta m \lambda/2) (\mL_t(\btheta_t^{(j)}) - \mL_t(\btheta_0)) + 
    C_4 \eta m \lambda (\sigma_{\eps}^2 + R^2) +
    \frac{5}{\eta m \lambda} \tilde{C}^2 C_2^2 (\sigma_{\eps}^2 + \\ \quad & R^2)\tau^{2/3} \lambda^{-1} L^6  \log(m)
    \quad \text{ (by $\tau = \tilde{C} \sqrt{\frac{\sigma_{\eps}^2 + R^2}{m \lambda}}$)}
    \\
    & \leq 
    (1-\eta m \lambda/2) (\mL_t(\btheta_t^{(j)}) - \mL_t(\btheta_0)) + 
    C_4 \eta m \lambda (\sigma_{\eps}^2 + R^2) +
    C_5 \eta m \lambda (\sigma_{\eps}^2 + R^2) \\
    & \quad \text{ (by choice of $\tau$ in Lemma~\ref{lemma: bound for gradient norm})}
\end{aligned}
\end{equation}
where $C_4$ is a constant depends on $\delta_3$ and $\delta_4$ and $C_5$ depends on $\delta_2$, $\delta_3$ and $\delta_4$. Then by recursion,
\begin{equation}
\begin{aligned}
    \mL_t(\btheta_t^{(j+1)}) - \mL_t(\btheta_0) 
    \leq 
    \frac{C_6 \eta m \lambda (\sigma_{\eps}^2 + R^2)}{\eta m \lambda/2} = \tilde{C}_6 (\sigma_{\eps}^2 + R^2)
\end{aligned}
\end{equation}
where $C_6=C_4 + C_5$ and $\tilde{C}_6 = 2 C_6$. Recall that $\norm{\bbf_t(\btheta) - \by_t}_2^2 = 2t \mL_t(\btheta)- \frac{m\lambda}{2}\norm{\btheta}_2^2 \leq 2t \mL_t(\btheta)$, with some constant $C_7$ derived from $C_6$ and $C_4$, then we have
\begin{equation}
\begin{aligned}
    \norm{\bbf_{gnn, t}^{(j)}-\by_t}_2^2 
    \leq 2t \mL_t(\btheta_t^{(j)}) 
    & \leq 
    2t \tilde{C}_6 (\sigma_{\eps}^2 + R^2) + 2t \mL_t(\btheta_0) \\
    &= 
    2t \tilde{C}_6 (\sigma_{\eps}^2 + R^2) + 2t(\frac{1}{t} \norm{\by_t}_2^2 + \frac{m\lambda}{2}\norm{\btheta_0}_2^2) \\
    & \leq 
    C_7 t (\sigma_{\eps}^2 + R^2)
    \quad \text{ (by Lemma~\ref{lemma: parameter bound for proximal optimization })}
\end{aligned}
\end{equation}
which implies our result by setting $\delta_1=\delta_2=\delta_3=\delta_4=\delta/4$ where $\delta \in (0,1)$ is arbitrary small.

\end{proof}

\begin{lemma} [Parameter Bound for Proximal Optimization]
\label{lemma: parameter bound for proximal optimization }
Let $\{\tilde{\btheta}_t^{(j)}\}_{j=1}^J$ be the gradient descent update sequence of parameters of the following optimization,
\begin{equation}
\begin{aligned}
    \min_{\btheta} \frac{1}{2t} \sum_{i=1}^t (\langle \bcg(G_i;\btheta_0), \btheta - \btheta_0 \rangle - y_i)^2 + \frac{m\lambda}{2} \norm{\btheta}_2^2
\end{aligned}
\end{equation}
Then if $m \ge poly(L, \lambda^{-1}, \log(N/\delta))$ and learning rate $\eta \leq (\tilde{C} mL+m\lambda)^{-1}$ for some constant $\tilde{C}$. Then for some constant $C$ and for any $\forall t \in [T]$ and $\forall j \in [J]$, with probability at least $1-\delta \in (0,1)$,
\begin{equation}
\begin{aligned}
    \norm{\tilde{\btheta}_t^{(j)}}_2
    & \leq
    C \sqrt{\frac{\sigma_{\eps}^2 + R^2}{m \lambda }} \\
    \norm{\tilde{\btheta}_t^{(j)} - \btheta_0}_2
    & \leq
    C \sqrt{\frac{\sigma_{\eps}^2 + R^2}{m \lambda }} \\
    \norm{\tilde{\btheta}_{t}^{(j)} - \btheta_0 - \bar{\bU}_{t}^{-1} \bar{\bG}_t \by_t/m }_2
    & \leq
    C (2-\eta m \lambda)^j  \sqrt{\frac{\sigma_{\eps}^2 + R^2}{m \lambda }}
\end{aligned}
\end{equation}
for some constant $C$ which is independent of $m$ and $t$.
\end{lemma}
\begin{proof}
Denote $\mL_t(\btheta):=\frac{1}{2t} \sum_{i=1}^t (\langle \bcg(G_i;\btheta_0), \btheta - \btheta_0 \rangle - y_i)^2 + \frac{m\lambda}{2} \norm{\btheta}_2^2$ as the loss function in our proximal optimization. By Lemma~\ref{lemma: gradient descent norm bound}, with probability at least $1-\delta_1 \in (0,1)$ the Hessian of $\mL_t(\btheta)$ satisfies:
\begin{equation}
\begin{aligned}
    \bZeros  \prec \nabla^2 \mL_t = \bar{\bG}_t \bar{\bG}_t^{\top}/t + m \lambda \bI \preccurlyeq (\norm{\bar{\bG}_t}_F^2/t + m \lambda) \bI
    \preccurlyeq (C_1^2 m L + m \lambda) \bI
\end{aligned}
\end{equation}
which reveals that $\mL_t$ is strongly convex and $(C_1^2 m L + m \lambda)$-smooth. Thus if $\eta \leq (C_1^2 m L + m \lambda)^{-1}$, $\mL_t$ is a monotonically decreasing function:
\begin{equation}
\begin{aligned}
    \frac{1}{2t} \norm{\bar{\bG}_t^{\top} (\tilde{\btheta}_t^{(j)} - \btheta_0) - \by_t}_2^2 + \frac{m\lambda}{2} \norm{\tilde{\btheta}_t^{(j)}}_2^2
    \leq 
    \frac{1}{2t} \norm{\by_t}_2^2 + \frac{m\lambda}{2} \norm{\btheta_0}_2^2
\end{aligned}
\end{equation}
which indicates
\begin{equation}
\begin{aligned}
    \norm{\tilde{\btheta}_t^{(j)}}_2^2
    \leq &
    \frac{1}{t m \lambda} \norm{\by_t}_2^2 + \norm{\btheta_0}_2^2 \\
    \leq &
    \frac{1}{t m \lambda} (\norm{\bmu_t}_2^2 + \norm{\bepsilon_t}_2^2 + 2\norm{\bmu_t}_2 \norm{\bepsilon_t}_2) + \norm{\btheta_0}_2^2
\end{aligned}
\end{equation}
Note that the proximal optimization is optimization for ridge regression which has the closed form solution:
\begin{equation}
\begin{aligned}
    \btheta^{*} = \btheta_0 + \bar{\bU}_t^{-1} \bar{\bG}_t \by_t/m
\end{aligned}
\end{equation}
and $\tilde{\btheta}_t^{(j)}$ converges to $\btheta^{*}$ with the following rate:
\begin{equation}
\begin{aligned}
    \norm{\tilde{\btheta}_t^{(j+1)} - \btheta^{*}}_2^2 
    & = 
    \norm{\tilde{\btheta}_t^{(j)} - \eta \nabla \mL(\tilde{\btheta}_t^{(j)}) - \btheta^{*}}_2^2 \\
    & =
    \norm{\tilde{\btheta}_t^{(j)} - \btheta^{*}}_2^2 + 
    \eta^2 \norm{\nabla \mL(\tilde{\btheta}_t^{(j)})}_2^2 -
    2 \eta (\tilde{\btheta}_t^{(j)} - \btheta^{*})^{\top} \nabla \mL(\tilde{\btheta}_t^{(j)}) \\
    \text{ (by smoothness)}  \mbox{ }& \leq 
    \norm{\tilde{\btheta}_t^{(j)} - \btheta^{*}}_2^2 + 
    \eta^2 (C_1^2 m L + m \lambda)^2 \norm{\tilde{\btheta}_t^{(j)} - \btheta^{*}}_2^2 -
    2 \eta (\tilde{\btheta}_t^{(j)} - \btheta^{*})^{\top} \nabla \mL(\tilde{\btheta}_t^{(j)}) 
    \\
    \text{ (by convexity)} \mbox{ } & \leq 
    \norm{\tilde{\btheta}_t^{(j)} - \btheta^{*}}_2^2 + 
    \eta^2 (C_1^2 m L + m \lambda)^2 \norm{\tilde{\btheta}_t^{(j)} - \btheta^{*}}_2^2 +
    2 \eta (\mL(\btheta^{*}) - \mL(\tilde{\btheta}_t^{(j)}) ) 
    \\
    & \leq 
    2 \norm{\tilde{\btheta}_t^{(j)} - \btheta^{*}}_2^2 + 
    2 \eta (\mL(\btheta^{*}) - \mL(\tilde{\btheta}_t^{(j)}) )
    \quad \text{ (by $\eta \leq (C_1^2 m L + m \lambda)^{-1}$)} \\
    & \leq 
    2 \norm{\tilde{\btheta}_t^{(j)} - \btheta^{*}}_2^2 - 
    \eta m \lambda \norm{\tilde{\btheta}_t^{(j)} - \btheta^{*}}_2^2
    \quad \text{ (by $m \lambda$-strongly convexity)} \\
    &= 
    (2-\eta m \lambda) \norm{\tilde{\btheta}_t^{(j)} - \btheta^{*}}_2^2
\end{aligned}
\end{equation}
Therefore,
\begin{equation}
\begin{aligned}
    \norm{\tilde{\btheta}_t^{(j+1)} - \btheta^{*}}_2^2 
    & \leq
    (2-\eta m \lambda)^j \norm{\btheta_0 - \btheta^{*}}_2^2 \\
    & \leq 
    (2-\eta m \lambda)^j \frac{2}{m \lambda}(\mL(\btheta_0)) - \mL(\btheta^{*}))
    \quad \text{ (by $m \lambda$-strongly convexity)} \\
    & \leq 
    (2-\eta m \lambda)^j \frac{2}{m \lambda} \mL(\btheta_0) \\
    & =
    (2-\eta m \lambda)^j  \bigg(\frac{1}{t m \lambda }\norm{\by_t}_2^2 + \norm{\btheta_0}_2^2 \bigg) 
\end{aligned}
\end{equation}
Then combine with Lemma~\ref{lemma: bound for subgaussian noise norm} and $\norm{\bmu_t}_{2} \leq \sqrt{t} \norm{\mu}_{\mH} \leq \sqrt{t} R $, we have that with probability at least $1-\delta_2 \in (0,1)$,
\begin{equation}
\begin{aligned}
    \frac{1}{t m \lambda} \norm{\by_t}_2^2 
    \leq 
    \frac{1}{t m \lambda}
    (t R^2 + \norm{\bepsilon_t}_2^2 + 2 \sqrt{t} R \norm{\bepsilon_t}_2)
    \leq 
    \tilde{C}_1 (\sigma_{\eps}^2 + R^2)/ m \lambda
\end{aligned}
\end{equation}
where $\tilde{C}_1$ is some constant depends on $\delta_2$. Therefore, for any $\delta \in (0,1)$, set $\delta_1=\delta_2 = \delta/2$, with probability at least $1-\delta_2$,
\begin{equation}
\begin{aligned}
    \norm{\tilde{\btheta}_t^{(j)}}_2
    & \leq
    \tilde{C}_2 \sqrt{\frac{\sigma_{\eps}^2 + R^2}{m \lambda }} \\
    \norm{\tilde{\btheta}_t^{(j)} - \btheta_0}_2
    & \leq
    \tilde{C}_2 \sqrt{\frac{\sigma_{\eps}^2 + R^2}{m \lambda }}
\end{aligned}
\end{equation}
and 
\begin{equation}
\begin{aligned}
    \norm{\tilde{\btheta}_{t}^{(j)} - \btheta_0 - \bar{\bU}_{t}^{-1} \bar{\bG}_t \by_t/m }_2
    \leq
    (2-\eta m \lambda)^j  \tilde{C}_2 \sqrt{\frac{\sigma_{\eps}^2 + R^2}{m \lambda }}
\end{aligned}
\end{equation}
where $\tilde{C}_2$ is some constant depends on $\delta_2$ and $\norm{\btheta_0}_2$.

\end{proof}

\begin{lemma} [Gradient Descent Norm Bound]
\label{lemma: gradient descent norm bound}
Define $\bG_t^{(j)} := [\bcg(G_1;\btheta_t^{(j)}),...,\bcg(G_t;\btheta_t^{(j)}))] \in \BR^{p \times t}$ for the gradients in the $j$-th updates in GNN training (optimization of \eqref{eq: minimize l2 loss}) at round $t$. Also define $\bbf_{gnn, t}^{(j)} := [\fgnn(G_1;\btheta_t^{(j)}), ..., \fgnn(G_t;\btheta_t^{(j)})]^{\top} \in \BR^{t \times 1}$. Assume $\tau$ is set such that $\norm{\btheta_t^{(j)} - \btheta_0}_2 \leq \tau$ for all $t$ and $\forall j \leq J$. Suppose $m \ge poly(L, \lambda^{-1}, \log(N/\delta))$ where $\delta \in (0,1)$, then with probability at least $1-\delta$,
\begin{equation}
\begin{aligned}
    \norm{\bar{\bG}_t}_F
    & \leq
    C_1 \sqrt{t m L} \\
    \norm{\bG_t^{(j)}}_F
    & \leq
    C_1 \sqrt{t m L} \\
    \norm{\bar{\bG}_t - \bG_t^{(j)}}_F
    & \leq
    C_2 \tau^{1/3} L^{7/2} \sqrt{ t m \log(m)} \\
    \norm{\bbf_{gnn, t}^{(j)} - (\btheta_t^{(j)} - \btheta_0)^{\top} \bar{\bG}_t}_2 
    & \leq C_3 \tau ^{4/3} L^3 \sqrt{t m \log(m)}
\end{aligned}
\end{equation}
for some constant $C_1$, $C_2$, $C_3$ which does not depend on $m$ and $t$.
\end{lemma}
\begin{proof}
From Lemma~\ref{lemma: bound for gradient norm}, we can bounding the $\norm{\bcg(G;\btheta_0)}_2$ with probability at least $1-\delta \in (0,1)$, which provides the high probability upper bound for the Frobenius norm of $\bar{\bG}_t$:
\begin{equation}
\begin{aligned}
    \norm{\bar{\bG}_t}_F \leq \sqrt{t} \max_{i \in [t]} \norm{\bcg(G_i;\btheta_0)}_2 \leq 
    \frac{\sqrt{t}}{N} \max_{i \in [t]} \sum_{j \in \mV(G_i)} 
    \norm{\gmlp(\bh_j;\btheta_0)}_2
    \leq C_1 \sqrt{t m L}
\end{aligned}
\end{equation}
and the high probability upper bound for the Frobenius norm of $\bG_t^{(j)}$:
\begin{equation}
\begin{aligned}
    \norm{\bG_t^{(j)}}_F \leq \sqrt{t} \max_{i \in [t]} \norm{\bcg(G_i;\btheta_t^{(j)})}_2 
    \leq 
    \frac{\sqrt{t}}{N} \max_{i \in [t]} \sum_{j \in \mV(G_i)} 
    \norm{\gmlp(\bh_j;\btheta_t^{(j)})}_2
    \leq
    C_1 \sqrt{t m L}
\end{aligned}
\end{equation}
For the gradients difference, by Lemma~\ref{lemma: bound for gradient norm}, with probability at least $1-\delta$,
\begin{equation}
\begin{aligned}
    \norm{\bar{\bG}_t - \bG_t^{(j)}}_F 
    & \leq
    \sqrt{t} \max_{i \in [t]} \norm{\bcg(G_i;\btheta_0) - \bcg(G_i;\btheta_t^{(j)})}_2 \\
    & \leq
    \frac{\sqrt{t}}{N} \max_{i \in [t]} \sum_{j \in \mV(G_i)}  \norm{\gmlp(\bh_j;\btheta_0) - \gmlp(\bh_j;\btheta_t^{(j)})}_2 \\
    & \leq 
    C_2 \tau^{1/3} L^{7/2} \sqrt{tm\log(m)} 
\end{aligned}
\end{equation}
The last norm for difference between the GNN prediction and linearized prediction is bounded due to Lemma~\ref{lemma: bound for gradient norm}, with probability at least $1-\delta$,
\begin{equation}
\begin{aligned}
    \norm{\bbf_{gnn, t}^{(j)} - (\btheta_t^{(j)} - \btheta_0)^{\top} \bG_t^{(j)}}_2 
    & \leq \sqrt{t} \max_{i \in [t]}|\fgnn(G_i;\btheta_t^{(j)}) - (\btheta_t^{(j)} - \btheta_0)^{\top}\bcg(G_i;\btheta_0) |
    \\
    & \leq 
    \frac{\sqrt{t}}{N}
     \max_{i \in [t]} \sum_{j \in \mV(G_i)} 
     |\fmlp(\bh_j;\btheta_t^{(j)}) - (\btheta_t^{(j)} - \btheta_0)^{\top}\gmlp(\bh_j;\btheta_0)|
    \\
    & \leq C_3 \tau ^{4/3} L^3 \sqrt{t m \log(m)}
\end{aligned}
\end{equation}
\end{proof} 

\subsection{Lemmas for GNTK}
\label{appendix: lemmas for gntk}

\begin{lemma} [Approximation from GNTK]
\label{lemma: approximation from gntk}
Set $\delta \in (0,1)$ and 
\[
m = \Omega (L^{10} T^4 |\mG|^6 \rho_{min}^{-4} \log(L N^2 |\mG|^2/ \delta) ).
\]
Then with probability at least $1-\delta$, \\
(i) (Approximate Linearized Nerual Network) 
$\exists \btheta^{*}$ such that, for $\forall G \in \mG$
\begin{equation}
\begin{aligned}
    \mu(G) &= \langle \bcg(G;\btheta_0), \btheta^{*} \rangle \\
    \sqrt{m} \norm{\btheta^{*}}_2 &\leq \sqrt{2} R
\end{aligned}
\end{equation}
(ii) (Spectral Bound for Uncertainty Matrix $\bar{\bU}_t$ by GNTK)
\begin{equation}
\begin{aligned}
    \lambda_{max}(\bar{\bU}_t) &\leq \lambda + \frac{3}{2} \rho_{\max} \\
    \log\det( \lambda^{-1} \bar{\bU}_t)
    & \leq 
    \log\det( \bI_{|\mG|} + \lambda^{-1} t \bK) + 1
\end{aligned}
\end{equation}
\end{lemma}
\begin{proof}
In this proof, set $\delta_1=\delta_2=\delta/2$ where $\delta \in (0,1)$ is an arbitrary real value. Recall the definition of the true reward function $\mu: \mG \rightarrow \BR$ and the GNTK matrix $\bK \in \BR^{|\mG| \times |\mG|}$. We further define the vector of function values $\bmu \in \BR^{|\mG| \times 1}$ as well as the gradient matrix $\bar{\bG} \in \BR^{p \times |\mG|}$ on initialization $\btheta_0$. 
\begin{equation}
\begin{aligned}
    [\bK]_{ij} &= k(G^i, G^j) \quad \forall G^i, G^j \in \mG \\
    [\bmu]_i &= \mu(G^i) \quad \forall G^i \in \mG \\
    \bar{\bG}_{*i} &= \bcg(G^i; \btheta_0)
\end{aligned}
\end{equation}
\textbf{Proof for (i):}
By the connection between GNTK and NTK,
\begin{equation}
\begin{aligned}
    \norm{\bK - \bar{\bG}^{\top} \bar{\bG}/m }_F
    &=
    \sqrt{\sum_{i=1}^{|\mG|} \sum_{j=1}^{|\mG|} (k(G^i, G^j) - \bcg^{\top}(G^i;\btheta_0)\bcg(G^j;\btheta_0)/m)^2} \\
    &= 
    \sqrt{\sum_{i=1}^{|\mG|} \sum_{j=1}^{|\mG|} \bigg(\frac{1}{N^2} \sum_{u \in \mV(G^i)} \sum_{v \in \mV_{G^j}}(\kmlp(\bh_u^{G^i}, \bh_v^{G^j}) - \gmlp^{\top}(\bh_u^{G^i};\btheta_0) \gmlp(\bh_v^{G^j};\btheta_0)/m) \bigg)^2} \\
    & \leq
    \sqrt{\sum_{i=1}^{|\mG|} \sum_{j=1}^{|\mG|} \sum_{u \in \mV(G^i)} \sum_{v \in \mV(G^j)} (\kmlp(\bh_u^{G^i}, \bh_v^{G^j}) - \gmlp^{\top}(\bh_u^{G^i};\btheta_0) \gmlp(\bh_v^{G^j};\btheta_0)/m)^2}
\end{aligned}
\end{equation}
where $\mV_{G}$ denotes the vertice set of a graph $G$. By Lemma~\ref{lemma:ntk:concent}, when $m=\Omega(L^{10} N^4 |\mG|^4 \rho_{min}^{-4} \log(L N^2 |\mG|^2/\delta_1)$, then with probability at least $1-\delta_1/(N^2 |\mG|^2)$, $|\kmlp(\bh_u^{G^i}, \bh_v^{G^j}) - \gmlp^{\top}(\bh_u^{G^i};\btheta_0) \gmlp(\bh_v^{G^j};\btheta_0)/m| \leq \frac{\rho_{min}}{2 N |\mG|}$. Then apply union bound over all pairs $(\bh_u^{G^i}, \bh_v^{G^j})$, the following holds with probability at least $1-\delta_1$,
\begin{equation}
     \norm{\bK - \bar{\bG}^{\top} \bar{\bG}/m }_F \leq \rho_{min}/2
\end{equation}
which shows that 
\begin{equation}
\begin{aligned}
\label{eq: pd of G^2}
    \bar{\bG}^{\top} \bar{\bG}/m 
    & \succcurlyeq  
    \bK - \norm{\bK - \bar{\bG}^{\top} \bar{\bG}/m}_2 \bI_{|\mG|} \\
    &  \succcurlyeq
    \bK - \norm{\bK - \bar{\bG}^{\top} \bar{\bG}/m}_F \bI_{|\mG|} \\
    &  \succcurlyeq 
    \bK - \frac{\rho_{min}}{2} \bI_{|\mG|} \\
    &  \succcurlyeq
    \bK/2 \succ \bZeros
\end{aligned}
\end{equation}
Suppose $\bar{\bG}=\bP \bLambda \bQ^{\top}$ is the decomposition of $\bar{\bG}$ where $\bP \in \BR^{p \times |\mG|}$, $\bQ \in \BR^{|\mG| \times |\mG|}$ are unitary and $\bLambda \in \BR^{|\mG| \times |\mG|}$. By \eqref{eq: pd of G^2}, we know $\bLambda  \succ \bZeros$ with probability at least $1-\delta_1$. Now denote $\btheta^{*} = \bP \bLambda^{-1} \bQ^{\top} \bmu$ and it satisfies
\begin{equation}
\begin{aligned}
    & \bar{\bG}^{\top} \btheta^{*} = \bQ \bLambda \bP^{\top} \bP \bLambda^{-1} \bQ^{\top} \bmu = \bmu\\
    \Rightarrow &
    \mu(G) = \langle \bcg(G;\btheta_0), \btheta^{*} \rangle \quad \forall G \in \mG
\end{aligned}
\end{equation}
Moreover, the norm of $\btheta^{*}$ is also bounded:
\begin{equation}
\begin{aligned}
    \norm{\btheta^{*}}_2^2 = \bmu^{\top} \bQ \bLambda^{-2} \bQ^{\top} \bmu = \bmu^{\top} (\bar{\bG}^{\top} \bar{\bG})^{-1} \bmu \leq \frac{2}{m} \bmu^{\top} \bK^{-1} \bmu \leq \frac{2 R^2}{m}
\end{aligned}
\end{equation}
which completes our proof for (i).

\textbf{Proof for (ii):}
From the definition of $\bar{\bG}_t$, we have
\begin{equation}
\begin{aligned}
    \log\det( \bI_{|\mG|} + \lambda^{-1} \bar{\bG}_t^{\top} \bar{\bG}_t/m)
    & =
    \log\det \bigg( \bI_{|\mG|} + \sum_{i=1}^t \bcg(G_i;\btheta_0) \bcg^{\top}(G_i;\btheta_0) /(m \lambda) \bigg) \\
    & \leq
    \log\det \bigg( \bI_{|\mG|} + t \sum_{G \in \cup_{i=1}^t \mG_i} \bcg(G;\btheta_0) \bcg^{\top}(G;\btheta_0) /(m \lambda) \bigg) \\
    & \leq
    \log\det \bigg( \bI_{|\mG|} + t \sum_{G \in \mG} \bcg(G;\btheta_0) \bcg^{\top}(G;\btheta_0) /(m \lambda) \bigg)
    \quad \text{(by $\mG_t \in \mG$ for $\forall t \in [T]$)} \\
    & = 
    \log\det( \bI_{|\mG|} + t \bar{\bG}^{\top} \bar{\bG}/(m \lambda) ) \\
    & = 
    \log\det( \bI_{|\mG|} + t\bK/ \lambda + t(\bar{\bG}^{\top} \bar{\bG}/m - \bK)/ \lambda ) \\
    \text{(by concavity of $\log \det(\cdot)$)} \mbox{ } & \leq 
    \log\det( \bI_{|\mG|} + t\bK/ \lambda) + \langle ( \bI + t\bK/ \lambda)^{-1},  t(\bar{\bG}^{\top} \bar{\bG}/m - \bK)/ \lambda \rangle_{F}
    \\
    & \leq 
    \log\det( \bI_{|\mG|} + t\bK/ \lambda) + 
    \norm{(\bI_{|\mG|} + t \bK/ \lambda)^{-1}}_F 
    \norm{t(\bar{\bG}^{\top} \bar{\bG}/m - \bK)/ \lambda}_F \\
    & \leq 
    \log\det( \bI_{|\mG|} + t\bK/ \lambda) + 
    t \sqrt{|\mG|} \norm{(\bI_{|\mG|} + t\bK/ \lambda)^{-1}}_2 
    \norm{\bar{\bG}^{\top} \bar{\bG}/m - \bK}_F /\lambda \\
    & =
    \log\det( \bI_{|\mG|} + t \bK/ \lambda) + 
    \sqrt{|\mG|} (\lambda/t + \rho_{min})^{-1}
    \norm{\bar{\bG}^{\top} \bar{\bG}/m - \bK}_F 
\end{aligned}
\end{equation}
By Lemma~\ref{lemma:ntk:concent}, when $m=\Omega(L^{10} N^4 |\mG|^6 \rho_{min}^{-4} \log(L N^2 |\mG|^2/\delta_2)$, then with probability at least $1-\delta_2/(N^2 |\mG|^2)$, $|\kmlp(\bh_u^{G^i}, \bh_v^{G^j}) - \gmlp^{\top}(\bh_u^{G^i};\btheta_0) \gmlp(\bh_v^{G^j};\btheta_0)/m| \leq \frac{\rho_{min}}{ N |\mG|^{3/2}}$. Then apply union bound over all pairs $(\bh_u^{G^i}, \bh_v^{G^j})$, with probability at least $1-\delta_2$, $\norm{\bar{\bG}^{\top} \bar{\bG}/m - \bK}_F \leq \frac{\rho_{min}}{\sqrt{|\mG|}}$, which indicates that
\begin{equation}
\begin{aligned}
    \log\det( \bI_{|\mG|} + \lambda^{-1} \bar{\bG}_t^{\top} \bar{\bG}_t/m) 
    &\leq 
    \log\det( \bI_{|\mG|} + t \bK/ \lambda) + 
    \sqrt{|\mG|} (\lambda/t + \rho_{min})^{-1}
    \norm{\bar{\bG}^{\top} \bar{\bG}/m - \bK}_F \\
    & \leq 
    \log\det( \bI_{|\mG|} + t \bK/ \lambda) +  1
\end{aligned}
\end{equation}
Finally, with probability at least $1-\delta_1$,
\begin{equation}
\begin{aligned}
    \bar{\bG}^{\top} \bar{\bG}/m
    \preccurlyeq 
    \bK + \norm{\bK - \bar{\bG}^{\top} \bar{\bG}/m}_2 \bI_{|\mG|}
    \preccurlyeq
    \bK + \frac{\rho_{\max}}{2} \bI_{|\mG|}
    \preccurlyeq
    \frac{3}{2} \rho_{\max}\bI_{|\mG|}
\end{aligned}
\end{equation}
which indicates that $\lambda_{max}(\bar{\bU}_t) \leq \lambda + \frac{3}{2} \rho_{\max}$.
\end{proof}

\begin{lemma}
\label{lemma: approx GNTK rho max}
Fix $\delta \in (0,1)$. Then, for $m=\Omega(L^{10} |\Gc|^4 \eps^{-4} \log(L/\delta))$, with probability at least  $1-\delta$,
\[
|\rho_{\max} - \hat{\rho}_{\max}|  \le \eps.
\]
\end{lemma}

\begin{proof}
Let $m$ be as in Lemma~\ref{lemma:ntk:concent}.
 Recall that $\norm{\bh_u^G} = 1$ for all $u \in \mathcal{V}(G)$ and $G \in \Gc$, by construction. Let $N_i := |\mV(G^i)|$. Then, we have, with probability at least  $1-\delta$,
\begin{align*}
    &|k(G^i, G^j) - \hat{k}(G^i, G^j)|\\
    &\qquad \le  
    \frac{1}{N_i N_j} 
    \sum_{\substack{u \in \mV(G^i)\\v \in \mV(G^j)}} 
    \bigl|\kmlp(\bh^{G^i}_u, \bh^{G^j}_v) -  \gmlp(\bh^{G^i}_u;\btheta_0)^\top \gmlp(\bh^{G^j}_v;\btheta_0) / m \bigr| \le \eps
\end{align*}
by Lemma~\ref{lemma:ntk:concent}. Then
\[
\opnorm{\bK - \bKh} \le  \fnorm{\bK - \bKh} \le  |\Gc| \eps.
\]
Then, from Weyl's inequality, $|\rhomax - \hat{\rho}_{\max}| \le |\Gc|\eps$. Replacing $\eps$ with $\eps/ |\Gc|$ the result follows.
\end{proof}

\section{Supporting Lemmas}
\label{appendix: supporting lemmas}

\begin{lemma}
\label{lemma: algebra lemma}
Suppose $\ba$, $\bb$ are vectors and $\bA$ is a matrix. $c$ is assumed to be positive scalar. Then we have the following results: (i) $|\ba^{\top} \bA \bb| \leq \sqrt{\ba^{\top} \bA \ba} \sqrt{\bb^{\top} \bA \bb}$. (ii) $\ba^{\top} \bb + c \norm{\ba}_2^2 \geq -\norm{\bb}_2^2/4c$.
\end{lemma}

\begin{lemma}
\label{lemma: gaussian concentration}
Suppose $X \sim \mN(\mu, \sigma^2)$ and $\beta > 0$, then
\begin{equation}
    \BP(|X - \mu| \leq \beta \sigma) \geq 1-e^{-\beta^2/2}
\end{equation}
\end{lemma}

\begin{lemma}
\label{lemma: gaussian anti-concentration}
Suppose $X \sim \mN(\mu, \sigma^2)$ and $\beta > 0$, then
\begin{equation}
    \pr(X - \mu > \beta \sigma ) \geq \frac{e^{-\beta^2}}{4 \beta \sqrt{\pi}}
\end{equation}
\end{lemma}

\begin{lemma}
\label{lemma: bound for subgaussian noise norm}
Suppose $\bepsilon \in \BR^t$ is a subgaussian random vector with subgaussian constant $\sigma^2$, then
\begin{equation}
    \BE[\norm{\bepsilon}_2] \leq 4 \sigma \sqrt{t}
\end{equation}
and with probability at least $1-\delta$ for $\delta \in (0,1)$,
\begin{equation}
    \norm{\bepsilon}_2 \leq C \sigma \sqrt{t}. 
\end{equation}
where $C$ is some constant depending on $\delta$.
\end{lemma}

\begin{lemma}
\label{lemma: bound for epsilon norm}
(Theorem 1 \citep{chowdhury2017kernelized}) Let $\{\bx_t\}_{t=1}^{\infty}$ be an $\BR^d$-valued discrete time stochastic process that is predictable with respect to the filtration $\{\mF_t\}_{t=1}^{\infty}$. Let $\{\eps_t\}_{t=1}^{\infty}$ be a real-valued stochastic process and for any $\forall t$, $\eps_t$ is $\mF_t$-measurable and subgaussian with constant $R$ conditionally on $\mF_{t-1}$. Let $k: \BR^d \times \BR^d \rightarrow \BR$ be a symmetric positive-definite kernel. Then for any $\eta>0$, $\delta \in (0,1)$, with probability at least $1-\delta$,
\begin{equation}
\begin{aligned}
    \norm{\bepsilon_t}_{((\bK_t+\eta \bI_t)^{-1}+\bI_t)^{-1}}^2 
    \leq 
    R^2 \log \det ((1+\eta)\bI_t + \bK_t) + 2R^2 \log(1/\delta)
\end{aligned}
\end{equation}
where $\bepsilon_t := (\eps_1,...,\eps_t)^{\top} \in \BR^t$ and $\bK_t \in \BR^{t \times t}$ is a matrix with $[\bK_t]_{ij} = k(\bx_i, \bx_j)$, $1\leq i,j \leq t$.
\end{lemma}

\begin{lemma}[Theorem 3.1 \citep{arora2019exact}]
\label{lemma:ntk:concent}
Fix $\eps > 0$ and $\delta \in (0,1)$. Suppose a MLP $\fmlp(\cdot; \btheta)$ with ReLU activation has $L$ layers and width $m=\Omega(L^{10} \eps^{-4} \log(L/\delta))$. Then for any input $\bx$, $\bx^{\prime}$ such that $\norm{\bx}_2 \leq 1$, $\norm{\bx^{\prime}}_2 \leq 1$, with probability at least $1-\delta$,
\begin{equation}
\begin{aligned}
    |\kmlp(\bx, \bx^{\prime}) - \gmlp(\bx;\btheta_0)^\top \gmlp(\bx^{\prime};\btheta_0) / m| \leq \eps
\end{aligned}
\end{equation}
where $\kmlp$ is the neural tangent kernel associated with $\fmlp$ and $\gmlp(\,\cdot\,; \btheta_0) = \nabla \fmlp(\,\cdot\,; \btheta_0)$ .
\end{lemma}

\begin{lemma} [Lemma B.4/Lemma B.5/Lemma B.6 \citep{zhou2020neural} / Lemma C.4 \citep{zhang2020neural}]
\label{lemma: bound for gradient norm}
Suppose $\btheta$ is parameters for a MLP $\fmlp(\cdot; \btheta)$ with $L$ layers and width $m$ and this neural network $\fmlp(\cdot; \btheta)$ is trained via gradient descent with initialization $\btheta_0$, learning rate $\eta$ and $\ell_2$ regularization constant $\lambda$ in a mean squared loss. The input feature set is denoted as $\mX = \{\bx_i\}_{i \in [T]}$. Then there are positive constants $\{C_i\}_{i=1}^7$ such that for $\forall \delta \in (0,1)$, if $\tau$ satisfies
\begin{equation}
\begin{aligned}
\label{eq: technical condition for tau}
    \tau & \geq 
    C_1 m^{-3/2} L^{-3/2} 
    max((\log(T L^2/\delta))^{3/2}, (\log(m))^{-3/2})  \\
    \tau & \leq 
    \min(C_2 L^{-6}(\log(m))^{-3/2}, C_3 L^{-9/2}(\log(m))^{-3}, C_4 m^{3} \lambda^{9/2} \eta^{3} L^{-9} (\log(m))^{-3/2})
\end{aligned}
\end{equation}
then with probability at least $1-\delta$, for $\norm{\btheta - \btheta_0}_2 \leq \tau$ and $\norm{\btheta^{\prime} - \btheta_0}_2 \leq \tau$, for $\forall \bx \in \mX$, we have 
\begin{equation}
\begin{aligned}
    \norm{\gmlp(\bx;\btheta) - \gmlp(\bx;\btheta_0)}_2 \leq C_5 \sqrt{\log(m)} \tau^{1/3} L^3 \norm{\gmlp(\bx;\btheta_0)}_2
\end{aligned}
\end{equation}
and 
\begin{equation}
\begin{aligned}
    |\fmlp(\bx; \btheta) - \fmlp(\bx; \btheta^{\prime})- \langle \gmlp(\bx,\btheta^{\prime}), \btheta-\btheta^{\prime} \rangle | \leq
    C_6 \tau^{4/3} L^3 \sqrt{m \log(m)}
\end{aligned}
\end{equation}
and 
\begin{equation}
\begin{aligned}
    \norm{\gmlp(\bx;\btheta)}_2 \leq C_7 \sqrt{mL}.
\end{aligned}
\end{equation}
\end{lemma}


\section{Supplement to Experiments}
\label{appendix: supplement to experiments}

\subsection{Data Generation}
\vspace{-2mm}
We use synthetic data environments for our experiments. The datasets are generated from two different random graph models and three different reward function generating models. 
The random graph models are Erd\"{o}s--R\'{e}nyi random graph model and random dot product graph model.
We use a linear model, Gaussian process with GNTK model, Gaussian process with representation kernel to generate our reward function.
In all data environments, the feature dimension is set as $d=10$. For any synthetic graph, all entries of the associated feature matrix $\{\bX_{ji}\}_{j \in [N], i \in [d]}$ are i.i.d from a standard Gaussian distribution. The noisy reward is assumed to have standard deviation $\sigma_{\eps} = 0.01$. 
All performance curves in our empirical studies show an average of over $10$ repetitions with a standard deviation of the corresponding bandit problem with horizon $T=1000$. Our experiment assumes the graph domain is fully observable, $\mG_t = \mG$ for all $t\in[T]$. 
We experiment four graph size $|\mG| \in \{10, 50, 100, 200\}$ in the random dot product graphs with $N=100$ and representation kernel.

\subsubsection{Random Graph}
\vspace{-2mm}
\textbf{Erd\"{o}s--R\'{e}nyi Random Graphs.}  Erd\"{o}s--R\'{e}nyi random graphs are generated by edge probability $p$ and number of nodes $N$. 
Set the graph has $N$ nodes and for any node pair $(i,j) \in [N]^2$, there is an edge linking $i$ and $j$ with probability  $p$. 
We investigate $p \in \{0.2, 0.4, 0.6, 0.8\}$ and $N \in \{10, 50, 100, 500\}$ in our experiment. 
Including $3$ types of reward function generating and $4$ sizes of graph space $\mG$, there are $192$ combinations of datasets of Erd\"{o}s--R\'{e}nyi random graph environments.

\textbf{Random Dot Product Graphs.} Random dot product graphs are generated by modeling the expected edge probabilities as the function of the inner product of features. In our experiment, we set the latent embeddings observed as features, i.e. $X_{i*}$ is the latent embedding of node $i$. Formally, the edge probability for node $i$ and $j$ is generated by $p_{ij} = \sigmoid (\bX_{i*}^{\top} \bX_{j*})$. We also investigate $N \in \{10, 50, 100, 500\}$. 
Including $3$ types of reward function generating and $4$ sizes of graph space $\mG$, there are $48$ combinations of datasets of random dot product graph environments. 

\subsubsection{Reward Function Generation}
\vspace{-2mm}
\textbf{Linear Model.}
We generate a true parameter $\btheta^{*} \in \BR^d$ whose elements are i.i.d standard Gaussian. Then the true reward mean is $\mu(G) = \ip{\btheta^{*}, \bar{\bh}^G}$
where $\bar{\bh}^G = \sum_{i=1}^N \bh^G_i/N$. 

\textbf{Gaussian Process with GNTK.}
We also use Gaussian process and Graph Neural Tangent Kernel(GNTK) as introduced from experiment in \citep{kassraie2022graph}. We approximately construct the GNTK matrix $\bK$ by the empirical GNTK matrix $\bKh \in \BR^{|\mG| \times |\mG|}$ whose entries are $\bKh_{ij} = \frac{1}{m} \langle \bcg(G^i;\btheta_0), \bcg(G^{j};\btheta_0) \rangle$ for any $G^i, G^j \in \mG$.
We use this empirical GNTK matrix $\bKh$ as the covariance matrix of prior, i.e, $\mN(0, \bK^{gntk})$  and use $\{(G, y_G)\}_{G \in \mG}$ where $\{y_G\}_{G \in \mG}$ are i.i.d from $\mN(0,1)$ as our training data. To train this Gaussian process model, we use negative log-likelihood loss with Adam optimizer with learning rate $0.01$ and $30$ epochs. The true reward means are sampled from the posterior in this Gaussian process.

\textbf{Gaussian Process with Representation Kernel.}
For the Gaussian process with representation kernel, we trained a GNN for a graph property prediction task and used the mean pooling over all nodes of the last layer representations as the graph representation. In our experiment, we utilize the average degree prediction as our task. That is, suppose outcome is $d_G = \frac{1}{N} \sum_{j=1}^N \degree(j)$ and train GNN in \eqref{eq: GNN} to predict this outcome.
Then denote the last layer representation as $\bar\bh^{G}_{\text{rep}} = \frac{1}{N} \sum_{j=1}^N f^{(L-1)}(\bh^G_j)$. Then we define the representation kernel as the inner product of the graph representations $k_{\text{rep}}(G, G') := \ip{\bar\bh^{G}_{\text{rep}}, \bar\bh^{G'}_{\text{rep}}}$.
The associated kernel matrix is denoted as $\bK^{rep} \in \BR^{|\mG| \times |\mG|}$ with entries $\{k^{rep}(G, G^{\prime}) \}_{G, G^{\prime} \in mG}$. In this Gaussian process, we sample the true reward means by $\{\mu(G)\}_{G \in \mG} \sim \mN(\bZeros, \bK^{rep})$.  To train this Gaussian process model, we use MSE loss with Adam optimizer with learning rate $0.01$ mini-batch size $2$ and $30$ epochs.

\vspace{-2mm}
\subsection{Algorithms Set Up}
\vspace{-2mm}
We provide the practical details and set up on our proposed algorithms and baseline algorithms.

\textbf{Algorithms. } 
We investigate $3$ GNN-based bandit algorithms (\texttt{GNN-TS}, \texttt{GNN-UCB} and \texttt{GNN-PE}) and $3$ corresponding NN-based bandit algorithms (\texttt{NN-TS}, \texttt{NN-UCB} and \texttt{NN-PE}).  All algorithms in our work use the loss function \eqref{eq: minimize l2 loss} which is different from previous work. All gradients used for in our experiments are $\bcg(G;\btheta_{t})$ not $\bcg(G;\btheta_{0})$ unless special stated. 
In addition, in order to show the benefit of considering the graph structure, we include \texttt{NN-UCB}, \texttt{NN-TS}, \texttt{NN-PE} as our baselines. For this NN-based algorithm, we ignore the adjacency matrix for a graph (assume $\bA = \bI$), and pass through the model in \eqref{eq: MLP} and \eqref{eq: GNN} by $\bh^G_i = \bX_{i*}$. 
For \texttt{GNN-TS}, we tuned the exploration scale with grid search on $\nu \in \{0.01, 0.1, 1.0, 10.0\}$ and \texttt{NN-TS} follows the same value.
For \texttt{GNN-UCB}, we tuned the hyperparameter with grid search on $\beta \in \{0.01, 0.1, 1.0, 10.0\}$ and \texttt{NN-UCB} follows the same value.
For \texttt{GNN-PE}, we tuned the hyperparameter with grid search on $\beta \in \{0.01, 0.1, 1.0, 10.0\}$ and \texttt{NN-PE} follows the same value.
All the hyperparameter tuning is performed in Erd\"{o}s--R\'{e}nyi random graphs with $p=0.4$, $N=50$, $|\mG| = 100$ and Gaussian process with GNTK for all the Erd\"{o}s--R\'{e}nyi random graphs settings and random dot product graphs with $50$ nodes and $|\mG| = 100$ and Gaussian process with GNTK for all the random dot product graphs settings.

\textbf{Neural Networks. } 
The MLPs in our experiments have $2$ layers ($L=2$) and width $m=512$. We use SGD optimizer with mini-batch size $5$ and $30$ epochs. Learning rates ($\eta$) we tuned from and the regularization hyperparameters $\lambda$ we tuned from $\{10^{-1}, 10^{-2}, 10^{-3}, 10^{-4}\}$. Initialization for the trainable GNN parameter $\btheta$ satisfies the condition $\fgnn(G;\btheta_0)=0$ for all $G \in \mG$, which is handle by the treatment in~\cite{kassraie2022neural}. Suppose the initialization is $\btheta_0$. 
The matrix inversion in the algorithms is approximated by diagonal inversion across all policy algorithms. 

\subsection{Experiments on Scalability ($|\mG|$)}
\vspace{-2mm}
We set the size of the graph domain to $|\mG|=100$ in Figure~\ref{figure: general performance} and we experiment across different sizes $|\mG| \in \{10, 50, 100, 200\}$ to check the scalability of the algorithms. 
Figure~\ref{figure: graph size plot} shows that given a fixed horizon length, larger $|\mG|$ leads to a harder bandit problem. It also shows that  \texttt{GNN-TS} can achieve top performance across all algorithms in all scales of the graph space. This empirical observation shows that \texttt{GNN-TS} is robust to the scalability of the action space, supporting our theoretical justification in Section~\ref{sec: regret bound}. 
\begin{figure}[t!]
    \centering
    \includegraphics[width=\textwidth]
    {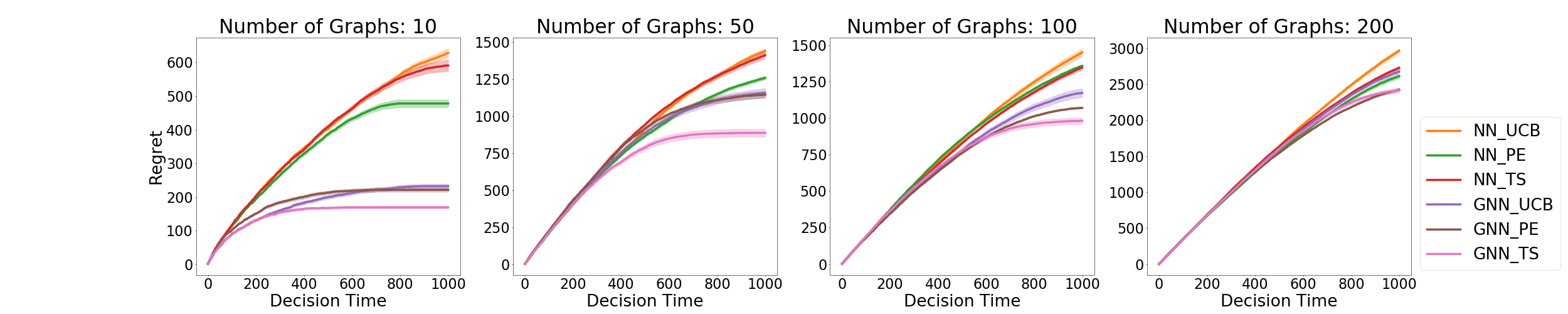}
    \caption{
    Competitive performance of  
    \texttt{GNN-TS} is consistent across different sizes of graph space. } 
    \label{figure: graph size plot}
\vspace{-4mm}
\end{figure}

\subsection{Effect of $m$ and Initial Gradients}

\begin{figure}[t!]
    \centering
    \includegraphics[width = 8cm, height = 3.8cm]
    {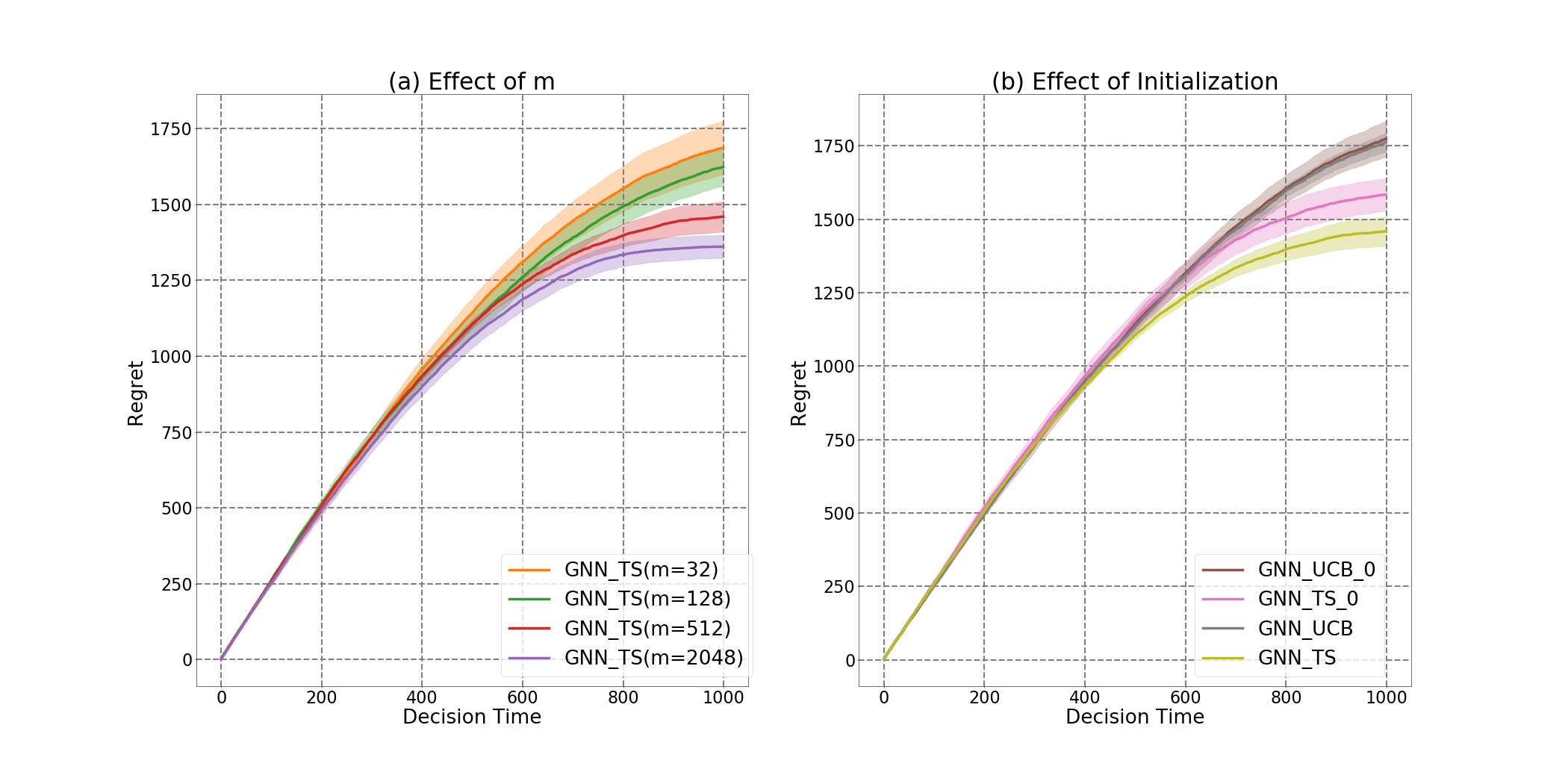}
    \caption{Increasing $m$ can improve the performance of \texttt{GNN-TS} and no improvement of using $\bcg(G_t;\btheta_0)$.} 
    \label{figure: effect plot}
\vspace{-4mm}
\end{figure}

Our regret analysis depends on the assumption that the width of the neural network $m$ must be large enough. We conduct an experiment to observe the effect from the width which is chosen from $\{32, 128, 512, 2048\}$. As some previous works on Neural bandit use the gradients at initialization ($\bcg(G_t;\btheta_0)$) for uncertainty calculation \citep{zhou2020neural, kassraie2022graph} while some works use $\bcg(G_t;\btheta_{t-1})$ which aligns with ours \citep{zhang2020neural}. Formally, instead of the update of uncertainty estimate in \eqref{eq: update U}, using initial gradient means performing the following
\begin{equation}
    \bar{\sigma}^2_{t}(G) = 
    \frac1{m}
    \norm{\bcg(G; \btheta_{0})}_{\bar{\bU}_{t}^{-1}}^2 ,\qquad
    \bar{\bU}_t = \bar{\bU}_{t-1} + \bcg (G_{t}; \btheta_{0}) \bcg(G_{t}; \btheta_{0})^{\top}/m.
\end{equation}
Part (a) of Figure \ref{figure: effect plot} reflects that the wider MLP has better performance which matches our expectation. Moreover, part (b) of Figure \ref{figure: effect plot} reflects that there are no benefits from setting gradients used in algorithms to be the initial gradients for all $t \in [T]$. One small final observation is that the effects of $m$ and initialization are not strong.

\subsection{Additional Figures and Tables}

\subsubsection{Results for Erd\"{o}s--R\'{e}nyi Random Graphs.}
For better visualization of the $192$ synthetic data environments using Erd\"{o}s--R\'{e}nyi random graphs, we summarised the result in Table~\ref{table: ER results}. The metrics are relative regret and top rate, which are defined based on regret as follow. The relative regret of one algorithm in one data environment is defined as 
\begin{equation}
    \text{Relative Regret:} \tilde{R}^{\text{alg, env}}  
    = 
    \frac{R_T^{\text{alg, env}} }{\max_{\text{alg}}  R_T^{\text{alg, env}} }
\end{equation}
where $R_T^{\text{alg, env}}$ is the cumulative regret of algorithm alg, and data environment env. 

We define the top rate for the policy in algorithm as the number of times such that the policy achieve the least two cumulative regret $R_T$. The denomnator is the number of total trails, which is the $1920$, the $10$ repetition and $192$ combinations of ER environments. The top rate of one algorithm is defined as
\begin{equation}
    \text{Top Rate:} \alpha_{\text{alg}}  
    = \frac{\text{\# times \text{alg} achieves "Top 2"}}{\text{\# trails}}.
\end{equation}

\begin{table}[ht]
\centering
\begin{tabular}{p{4cm} p{1.5cm}p{1.5cm}p{1.5cm}p{1.5cm}p{1.5cm}p{1.5cm}} 
\hline
& \texttt{NN-UCB} & \texttt{NN-PE} & \texttt{NN-TS} & \texttt{GNN-UCB} & \texttt{GNN-PE} & \texttt{GNN-TS} \\
\hline
Top Rate ($\alpha_{\text{alg}}$) &
$0.0 \%$ & $1.6\%$ & $0.0 \%$ & $9.4 \%$ & $90.6 \%$ & $\textbf{98.4 \%}$ \\
Relative Regret ($\tilde{R}^{\text{alg, env}}$) &
$0.994 (0.02)$ & $0.891 (0.06)$ & $0.943(0.05)$ & $0.762(0.15)$ & $0.690(0.14)$ & $\textbf{0.595(0.16)}$ \\
\hline
\end{tabular}
\caption{
Results on Erd\"{o}s--R\'{e}nyi random graphs. $192$ data environments with $10$ repetitions. }
\label{table: ER results}
\end{table}

\subsubsection{Results for Random Dot Product Graphs}
\vspace{-2mm}
We provide the experiment results for regret on all random dot product graph settings. In thee plots, different rows represents different sizes of the graph space ($|\mG|$) and columns represents the choices of the number of nodes in the graph ($N$).

\begin{figure}[ht]
    \centering
    \includegraphics[height = 6.1cm]
    {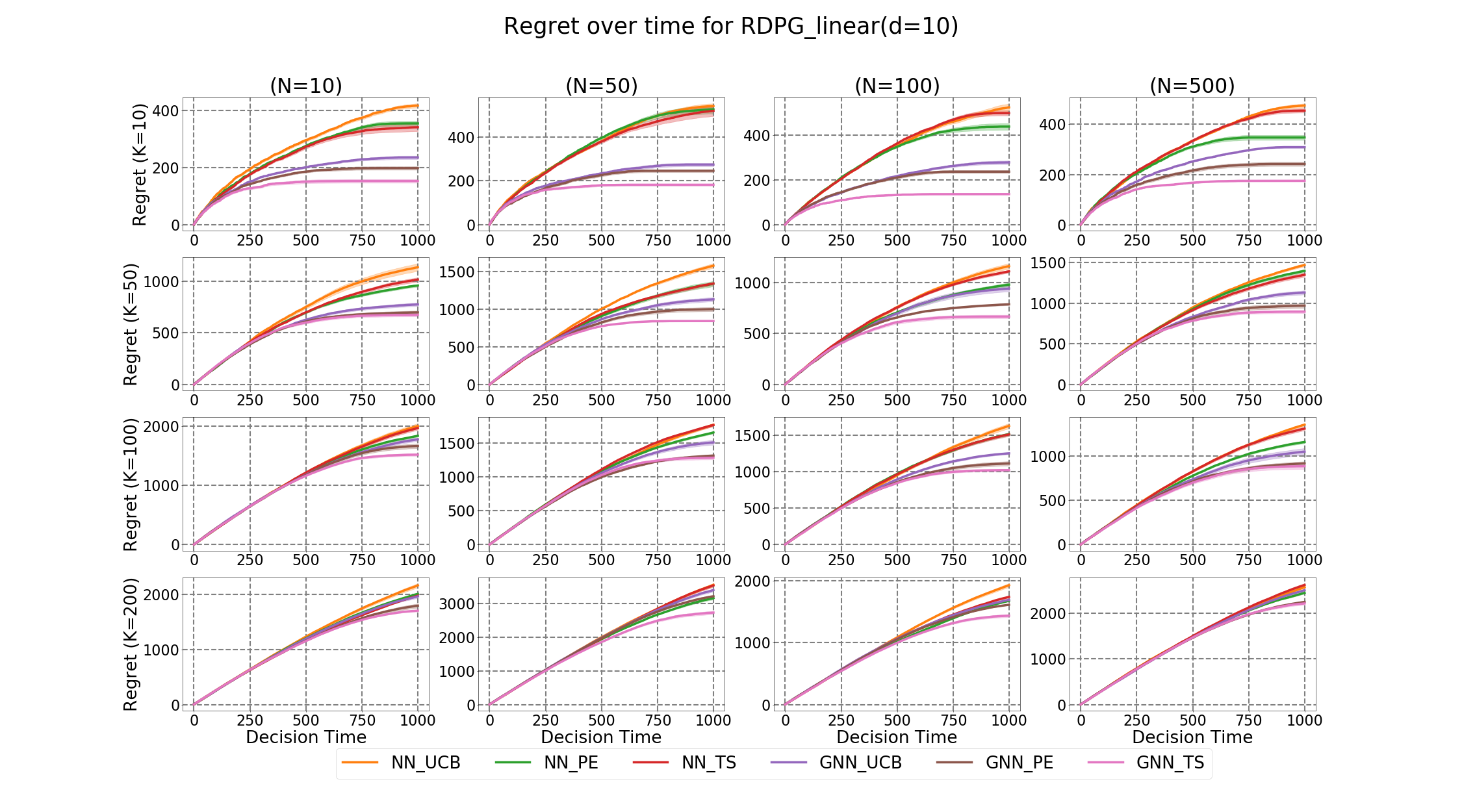}
    \caption{Random Dot Product Graphs with linear reward.} 
    \label{figure: RDPG_arm_gp_linear(d=10)}
\vspace{-6mm}
\end{figure}

\begin{figure}[ht]
    \centering
    \includegraphics[height = 6.1cm]
    {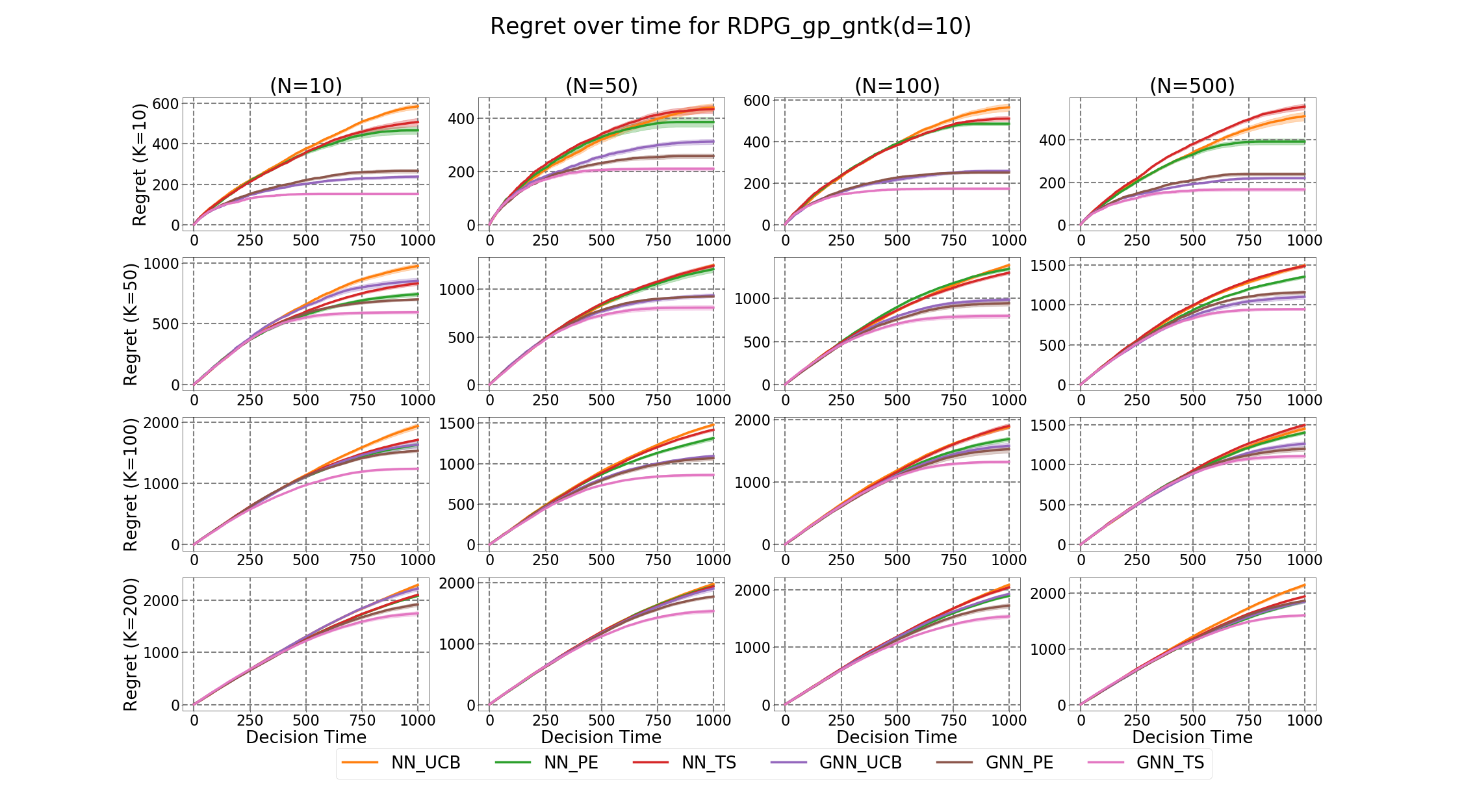}
    \caption{Random Dot Product Graphs with GP and GNTK for reward.} 
    \label{figure: RDPG_arm_gp_gntk(d=10)}
\vspace{-6mm}
\end{figure}

\begin{figure}[ht]
    \centering
    \includegraphics[height = 6.1cm]
    {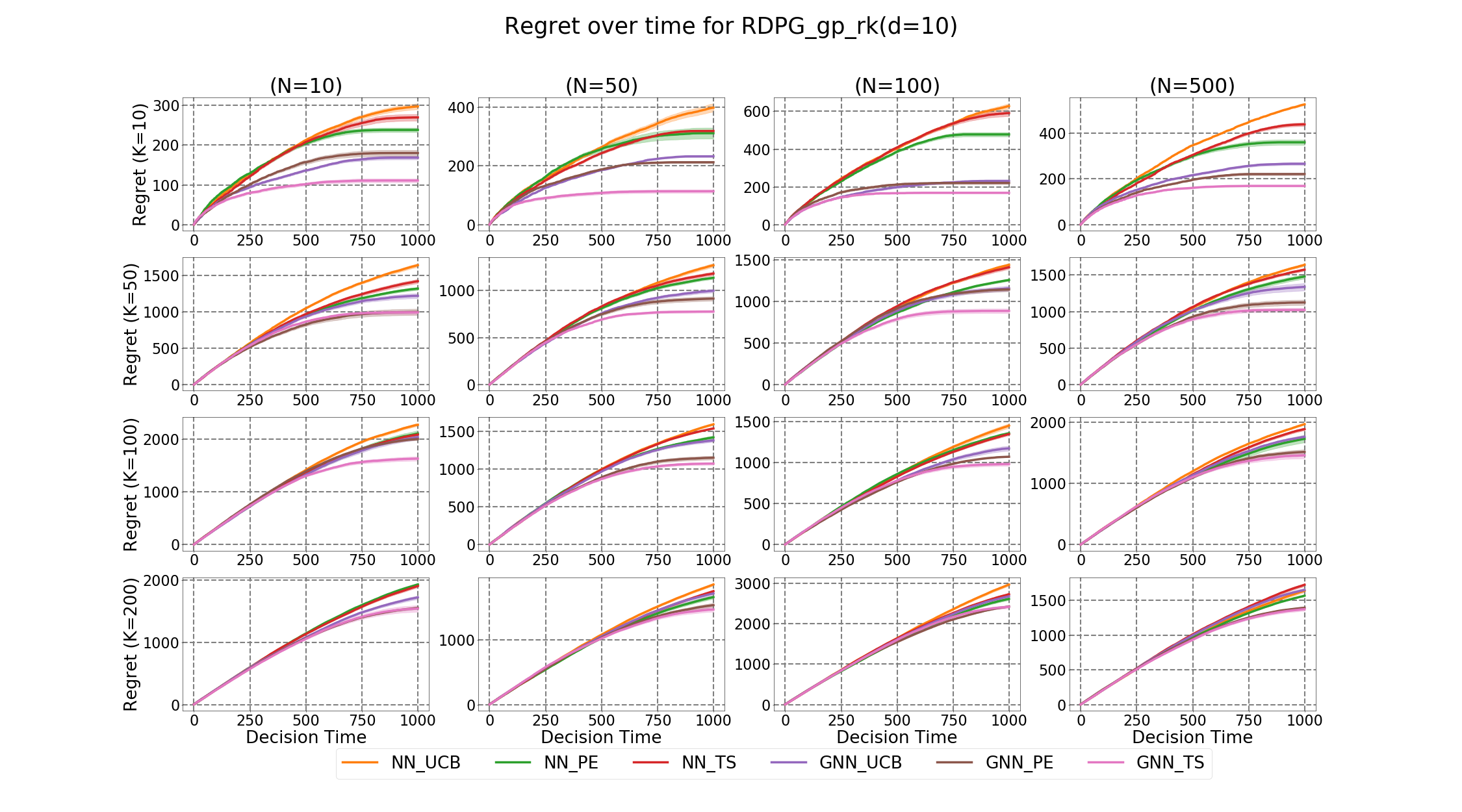}
    \caption{Random Dot Product Graphs with GP and representation kernel for reward.} 
    \label{figure: RDPG_arm_gp_rk(d=10)}
\vspace{-6mm}
\end{figure}

\end{document}